%% file: arxiv.tex
\documentclass[10pt,twocolumn,letterpaper]{article}

\usepackage{cvpr}              

\usepackage{url}
\usepackage[utf8]{inputenc} 
\usepackage[T1]{fontenc}    
\usepackage{url}            
\usepackage{booktabs}       
\usepackage{amsfonts}       
\usepackage{nicefrac}       
\usepackage{microtype}      
\usepackage{xcolor}         

\usepackage{soul}
\usepackage{url}
\usepackage{mathrsfs}
\usepackage{color}
\usepackage{enumitem}
\usepackage{bm}
\usepackage{bbm}
\usepackage{comment}
\usepackage{algorithm}
\usepackage{algorithmic}
\usepackage{amsmath}
\usepackage{amssymb}
\usepackage{mathtools}
\usepackage{amsthm}
\theoremstyle{plain}

\newtheorem{theorem}{Theorem}
\newtheorem{proposition}{Proposition}
\newtheorem{lemma}{Lemma}

\theoremstyle{definition}
\newtheorem{definition}{Definition}
\newtheorem{assumption}{Assumption}
\newtheorem{remark}{Remark}

\theoremstyle{definition}
\usepackage{multirow}
\usepackage{booktabs}            
\usepackage[table]{xcolor}       
\usepackage{tabularx}            
\usepackage{array}               
\usepackage{threeparttable}      
\usepackage{caption}             

\title{Private Variance Reduction for Nonconvex Optimization under Distributional Uncertainty}
\definecolor{cvprblue}{rgb}{0.21,0.49,0.74}
\usepackage[pagebackref,breaklinks,colorlinks,allcolors=cvprblue]{hyperref}

\def\confName{CVPR}
\def\confYear{2027}

\author{%
Difei Xu\textsuperscript{1,2}\thanks{These authors contributed equally.}\quad
Meng Ding\textsuperscript{2,3}\footnotemark[1]\quad
Zebin Ma\textsuperscript{1,2}\quad
Huanyi Xie\textsuperscript{1,2}\\
Youming Tao\textsuperscript{2,4}\quad
Aicha Slaitane\textsuperscript{1,2}\quad
Di Wang\textsuperscript{1,2}\\[0.5em]
\textsuperscript{1}Provable Responsible AI and Data Analytics (PRADA) Lab\\
\textsuperscript{2}King Abdullah University of Science and Technology\\
\textsuperscript{3}The University of Massachusetts Boston\\
\textsuperscript{4}Technische Universit\"at Berlin
}

\begin{document}
\maketitle
\input{sec/0_abstract}    
\input{sec/1_intro}

\input{sec/2_formatting}
\input{sec/3_finalcopy}

{
    \small
    \bibliographystyle{ieeenat_fullname}
    \bibliography{main}
}
\input{author-kit-CVPR2027-v1-latex/sec/appendix}

\end{document}

%% file: sec/0_abstract.tex
\begin{abstract}
Distribution shifts, group imbalances, and adversarial perturbations are common in real world deployments and can substantially degrade the performance of traditional Empirical Risk Minimization (ERM). Distributionally Robust Optimization (DRO) provides a principled response to these challenges by optimizing the worst-case expected loss over an uncertainty set of distributions. At the same time, the sensitive information involved in training DRO models makes Differential Privacy (DP) essential for protecting training data. Despite this need, DP DRO has received considerably less attention than classical DP ERM, owing to its minimax optimization structure with uncertainty constraints.
To address this gap, we present a comprehensive study of empirical DP DRO with $\psi$-divergence and nonconvex loss. We first consider DP DRO with KL divergence and a fixed divergence penalty. By exploiting its compositional finite sum optimization structure, we develop DP Recursive Spider, which combines recursive gradient tracking with a single sample estimator of the inner expectation. We then extend our study to a broader class of $\psi$-divergences satisfying local regularity conditions and develop DP Double Spider through a  dual formulation. Under bounded loss and regularity assumptions, both methods achieve a utility bound of $\widetilde{\mathcal O}((\frac{\sqrt{d}}{n\varepsilon})^{2/3})$ in terms of the expected gradient norm, where $n$ is the data size, $d$ is the model dimension, and $\widetilde{\mathcal O}$ suppresses logarithmic factors in $n$ and $1/\delta$. Experiments demonstrate that the proposed methods outperform existing approaches for DP minimax optimization.
\end{abstract}

%% file: sec/1_intro.tex
\section{Introduction}
\label{sec:intro}




With the rapid expansion of machine learning in real-world applications, the demand for robust model training has become increasingly critical. 
Conventional Empirical Risk Minimization (ERM) aims to minimize the expected loss under the empirical distribution of the training data, with the hope that the resulting model generalizes well to unseen test data. 
However, in practice, ERM often struggles when the training and test distributions do not align. Such mismatches are pervasive, including domain shifts in transfer and domain adaptation tasks \citep{blitzer2006domain}, imbalanced training datasets where fairness for underrepresented groups is essential \citep{sagawa2019distributionally}, and adversarial perturbations that deliberately mislead deployed models~\citep{goodfellow2014explaining,sinha2017certifying,madry2017towards}. 
In these settings, models trained purely with ERM typically exhibit severe performance degradation, limiting their reliability in safety-critical or fairness-sensitive applications. 
To overcome these limitations, Distributionally Robust Optimization (DRO) has emerged as a powerful alternative \citep{ben2013robust,shapiro2017distributionally,rahimian2019distributionally}. 
By optimizing performance against the worst-case distribution within a specified uncertainty set, DRO provides models with resilience to distribution shifts, fairness violations, and adversarial attacks. This worst-case perspective has made DRO a central tool in building robust learning systems.

Meanwhile, as the training of DRO often involves private data, safeguarding sensitive information has become a critical priority. Differential privacy (DP)~\citep{dwork2014algorithmic} now occupies a central position in data analysis, driven by the growing demand for principled protection guaranties. 
The foundational work \cite{dwork2006calibrating} established the modern framework for DP, and its incorporation into empirical risk minimization (ERM) has progressed substantially. 
Early research focused primarily on convex optimization, delivering strong privacy guarantees alongside efficient learning, for example, \citep{bassily2014private, bassily2019private, bassily2021differentially, choquette2024optimal, feldman2020private, su2023differentially, tao2022private, wang2020empirical, wang2020differentially, wang2017differentially}. 
More recently, attention has shifted toward non-convex settings. Notable advances such as \citep{wang2019differentially, zhou2020private, xiao2023theory, arora2023faster} tighten error bounds for non-convex optimization while balancing privacy and utility in stochastic gradient methods. 
However, relatively few studies have focused on (finite-sum) DRO, which involves minimax optimization over an uncertainty set. 

Recently, several works have begun to investigate DP-DRO~\cite{zhou2024differentially,selvi2025differential}. 
However, these studies either consider only a special classes of DRO (such as group DRO~\cite{zhou2024differentially}) or propose methods that are not scalable to large datasets~\citep{selvi2025differential}. 
DRO with $\psi$-divergence, a widely used framework in machine learning, has not yet been systematically studied under DP. 
This gap motivates the present work. 
Although DRO can be formulated as a minimax problem, it differs fundamentally from classical minimax optimization. 
As a result, recent advances in DP minimax optimization cannot be directly applied to analyze or solve DP-DRO due to their assumptions on the loss function or the absence of uncertainty sets.

To fill this gap, we provide the first systematic study of DP-(finite-sum) DRO with  non-convex loss. 
Specifically, we fist focus on DRO with KL divergence. Rather than introducing dual variables, the problem can be rewritten as a compositional minimization problem. Based on this, we develop  our follow our idea DP-Double Spider and propose DP Recursive-Spider. Theoretically, we show that the utility rate can be improved to \(\mathcal{O}\!\left(\bigl(\tfrac{\sqrt{d \log(1/\delta)}}{n \varepsilon}\bigr)^{\!2/3}\right)\), which matches the best-known result in DP-ERM with non-convex loss~\cite{arora2023faster}. It is notable that this is also the first result on DP compositional optimization.

Then, by establishing a primal--dual equivalence for distributionally robust optimization, we reduce the original minimax problem to an unconstrained minimization problem over model parameters and a dual variable. Building on this reformulation, we develop a novel private optimization method called DP Double-Spider, which is a differentially private and non-trivial extension of SPIDER~\citep{fang2018spider} from ERM to DRO.  DP Double-Spider alternates updates of the primal and dual variables using variance-reduced stochastic gradients perturbed with calibrated Gaussian noises, which outperforms classical DP-SGD-based methods such as DP-SGDA.  Theoretically, under mild assumptions, we show that the utility of the resulting private model, measured by the gradient norm, is upper bounded by $\mathcal{O}(  (\frac{\sqrt{d \log (1/\delta)}}{n \varepsilon})^{2/3})$, where $n$ is the data size and $d$ is the model dimension.





Our contribution can be summarized as below:
\begin{enumerate}

    \item \textbf{A generic DP-DRO framework}: For DRO with general $\psi$-divergence, based on its one primal-dual equivalent formulation,  we propose the Differentially Private Double SPIDER algorithm to deal with both the primal and dual variables simultaneously. We claim that the algorithm is $(\varepsilon , \delta )$-DP and the utility of our private model is $\mathcal{O}( (\frac{\sqrt{d \log (1/\delta)}}{n \varepsilon})^{2/3})$.
    \item \textbf{A DP method for DRO with KL-divergence}: For DRO with KL-divergence, we propose DP Recursive-Spider and show its private model achieves an utility bound of $\mathcal{O}((\frac{\sqrt{d \log(1/\delta)}}{n\varepsilon})^{2/3} )$.
    \item  We implement our algorithms on the constructed imbalanced dataset CIFAR10-ST, MNIST-ST and the inherently imbalanced dataset CelebA and   Fashion-MNIST. Our results demonstrate that our algorithms perform better than the previous baselines  across distinct datasets and privacy budgets, providing empirical support for our theoretical analysis.
\end{enumerate}

%% file: sec/2_formatting.tex

\section{Related Works}
\textbf{DRO}
Depending on how the uncertain variables are controlled, DRO formulations can be categorized as either constrained, imposing explicit limits on the uncertain variables or regularized, penalizing them through a regularization term in the objective \citep{levy2020large}. \cite{duchi2019variance} further showed that a DRO problem with quadratic regularization in constraint form is equivalent to minimizing the empirical loss augmented by a variance regularization term on itself.

Many primal–dual methods for minimax optimization \citep{duchi2021learning,levy2020large,wang2021sinkhorn,duchi2021statistics} assume convexity of the loss function, which restricts their applicability to the broader class of non-convex machine learning problems. For non-convex DRO, existing works \citep{qi2021online,qi2022stochastic,zhang2024large} typically impose boundedness or even stronger assumptions on the loss $\ell$, further limiting their practical relevance. The connection between compositional functions and DRO formulations has been recognized and exploited in prior work. \cite{dentcheva2017statistical} investigated the statistical estimation of compositional functionals with applications to conditional-value-at-risk (CVaR) measures, which are closely related to CVaR-constrained DRO, though without considering stochastic optimization algorithms. To the best of our knowledge, \cite{qi2021online} first applied stochastic compositional optimization to KL-regularized DRO problems, while \cite{qi2022stochastic} extended this framework to the more challenging KL-constrained setting.

\noindent\textbf{DP Minimax Optimization} 
As DRO can be regarded as a constrained minimax problem. Therefore, our work is related to the recent methods for DP minimax optimization~\cite{yang2022differentially,zhang2025improved,zhang2022bring,zhao2023differentially}. For example, \cite{yang2022differentially} analyzes DP-SGDA for stochastic minimax learning using an algorithmic-stability framework. They reach the near-optimal bound for convex-concave objectives using SGDA.  \cite{zhang2025improved} study differentially private stochastic minimax optimization in the nonconvex-strongly-concave (NC-SC) regime and give the first general results beyond convex/PL settings. However, without tailored algorithmic design, standard minimax optimization methods cannot be directly applied to DRO problems, rendering the two frameworks fundamentally incomparable. Under the DRO constraint, plain Stochastic Gradient Descent–Ascent (SGDA) fails in the non-convex setting, as it cannot guarantee convergence to an optimal point without the convex–concave structure \citep{hsieh2021limits} or additional structural assumptions on the optimization variables. 

\noindent\textbf{DP-DRO}  As we mentioned, currently there is no work on DP for finite-sum DRO with general $\psi$-divergence, which has been widely adopted in the deep learning community. Recent work on DP-DRO only considered a specific class. In detail, \cite{selvi2304differential} develops a class of mechanisms that enjoy non-asymptotic and unconditional optimality guarantees. \citep{zhou2024differentially} investigates worst-group risk minimization (i.e., group DRO) under differential privacy, proposing algorithms that minimize the maximum risk across multiple sub-populations while preserving privacy. It attempts to privatize the given non-private algorithm for convex loss functions. They reach near-optimal guarantees, introducing simple private algorithms. However, in practice, their methods cannot be applied to our continuous uncertainty set. 

However, there are two limitations that make these work unsuitable for our DRO problem. First, previous work only considers either convex-(strongly) concave or non-convex-strongly-concave case, while for DRO with $\psi$-divergence, it is a non-convex-concave problem.  Thus, previous methods cannot be used for analyzing our problem. Second, directly using previous methods to our problem requires performing the project onto the simplex in each update, which is prohibitively costly. In our work, we address these limitations and we further improve the utility bound for DRO compared to previous methods.

\section{Preliminaries}\label{preliminary}

\subsection{Differential Privacy}
\begin{definition}[Differential Privacy \citep{dwork2006calibrating}]\label{def:dp}
	Given a data universe $\mathcal{X}$, we say that two datasets $S,S'\subseteq \mathcal{X}$ are neighbors if they differ by only one entry, which is denoted as $S \sim S'$. A randomized algorithm $\mathcal{A}$ is $(\varepsilon,\delta)$-differentially private (DP) if for all neighboring datasets $S,S'$ and for all events $E$ in the output space of $\mathcal{A}$, the following holds
	$$\mathbb{P}(\mathcal{A}(S)\in E)\leqslant e^{\varepsilon} \mathbb{P}(\mathcal{A}(S')\in E)+\delta.$$ 
 If $\delta=0$, we call algorithm $\mathcal{A}$ $\varepsilon$-DP. 
\end{definition}
\begin{definition}
    Given a function $q: \mathcal{Z}\to \mathbb{R}^d$, we say $q$ has $\Delta_2(q)$ $\ell_2$-sensitivity if for any neighboring datasets $D,D'$ we have $\|q(D)-q(D')\|\leqslant \Delta_2(q)$.
\end{definition}
\begin{definition}
    Given any function $q : \mathcal{Z} \to \mathbb{R}^d$, the Gaussian mechanism is defined as $q(D) + \xi$ where $\xi \sim \mathcal{N}\left(0, \frac{2\Delta_2^2(q) \log(1.25/\delta)}{\varepsilon^2} \mathbf{I}_d\right)$. Gaussian mechanism preserves $(\varepsilon, \delta)$-DP for $0 < \varepsilon, \delta \leq 1$.
\end{definition}

\subsection{Distributionally Robust Optimization}
As we mentioned earlier, unlike Empirical Risk Minimization, there are different formulations of DRO for various applications. In this paper, we will focus on the fine-sum version of DRO with an uncertainty set, which is the most commonly used model in deep learning to handle distribution shift~\citep{levy2020large,jin2021non,qi2021online,sinha2017certifying}. Given an underlying data distribution $P_0$ and a loss $\ell$, DRO aims to find a model $x\in \mathbb{R}^d$ that minimizes the loss over all data that is close to the distribution $P_0$:
\begin{equation*}
    \min_{x\in \mathbb{R}^d} \max_{P\in \mathcal{U}(P_0)} \mathbb{E}_{\xi \sim P}[\ell(x; \xi)],  
\end{equation*}
where $\mathcal{U}(P_0)$ is an uncertainty set for $P_0$. Specifically, when $\mathcal{U}(P_0)=\{P: D_\psi(P, P_0)\leq \rho \}$ for some divergence $D_\psi$ and radius $\rho$, we refer to the problem as DRO with divergence $D_\psi$.  It includes the $\chi^2$-divergence, KL-divergence, and the Cressie--Read family, defined as
\begin{equation}\label{psidivergence}
  D_\psi(Q \| P_0) = \sum_{x \in \mathcal{X}} p_0(x) \psi \left( \frac{q(x)}{p_0(x)} \right),
\end{equation}
where $\psi$ is a non-negative convex function satisfying $\psi(1)=0$ and $\psi(t)=+\infty$ for all $t<0$.  

In this paper, we consider the empirical form of DRO. 
Specifically, given a training data $S=\{\xi_i\}_{i=1}^n$, an equivalent form of DRO with $\psi$-divergence can be formulated as the following constrained minimax optimization problem: 
\begin{equation}\label{penalized}
     \min_{x\in \mathbb{R}^d}\mathcal{F}(x):= \max_{p \in \Delta_n} \sum_{i=1}^{n} p_i \ell(x;\xi_i) - \lambda_0 D_\psi(p, \frac{1}{n}\mathbb{I}), 
\end{equation}
where $\frac{1}{n}\mathbb{I}$  is the  empirical distribution and $\Delta_n=\{\mathbf{p}\in \mathbb{R}^n:\sum_{i=1}^np_i=1,p_i \geqslant0\}$ denotes a $n$-dimensional simplex. The hyperparameter $\lambda_0$ is pre-specified and remains fixed during training.    

\begin{table*}[htbp]
\centering
\caption{Common divergences with $L$-smooth conjugates.}\label{conjugate_table} 
\begin{tabular}{c|c|c}
\toprule
\textbf{Divergence} & $\psi(t)$ & $\psi^*(t)$ \\
\midrule
$\chi^2$ & $\frac{1}{2}(t-1)^2$ & $-1 + \frac{1}{4}(t+2)^2$ \\
KL-regularized CVaR & $\mathbb{I}_{[0,\alpha^{-1})} + t\log(t) - t + 1, \alpha \in (0,1)$ & $\min(e^t, \alpha^{-1}(1 + t + \log(\alpha))) - 1$ \\
Cressie-Read & $\frac{t^k - tk + k - 1}{k(k-1)}, k \in \mathbb{R}$ & $\frac{1}{k}\left(((k-1)t + 1)_+^{\frac{k}{k-1}} - 1\right)$ \\
\bottomrule
\end{tabular}
\end{table*}

\begin{definition}[DP-DRO] In the problem of DP-DRO, given a training data $S$, we aim to develop an $(\varepsilon,\delta)$-DP algorithm  $\mathcal{A}$ for solving problem \eqref{penalized}. Moreover, we need to make the utility of the private model $\mathbb{E}[\|\nabla\mathcal{F}(\mathcal{A}(S))\|]$ to be as small as possible, where the expectation takes over the randomness of the algorithm. 
\end{definition}

Next, we will impose several assumptions on the problem for our theoretical utility analysis. 

\begin{definition}
    A function $f: \mathbb{R}^d \to \mathbb{R}$ is  $G$-Lipschitz continuous if $\forall x, y  \in \mathbb{R}^d$, $\left| f(w_1;x)-f(w_2;x) \right|  \leqslant G \left\| w_1-w_2 \right\|$, where $G>0$   is some finite constant.
\end{definition}
\begin{definition}
    A differentiable function $f: \mathbb{R}^d \to \mathbb{R}$ is  $L$-smooth if $\forall x, y \in \mathbb{R}^d$, we have $\left\| \nabla f(x)-\nabla f(y) \right\|  \leqslant L \left\| x-y \right\|$, where $L>0$  is some finite constant.
\end{definition}

It is noteworthy that Lipschitz and $L$-smoothness are commonly adopted in optimization research \citep{zhang2025revisiting,qi2021online,qi2022stochastic,zhou2024differentially,zhang2025improved}. However, as we will mention later, DRO with $\psi$-divergence may not satisfy smoothness but a more general property called generalized smoothness. 

\begin{definition} [Generalized $(L_0,L_1)$-smooth]
    A differentiable function $f: \mathbb{R}^d \to \mathbb{R}$ is genenralized $(L_0,L_1)$-smooth if for any $x,y \in \mathbb{R}^d$, we have that $\left\| \nabla _x f(x)-\nabla _x f(y) \right\| \leqslant (L_0+L_1 \left\| \nabla _xf(x) \right\| )\left\| x-y \right\| $, where $L_0,L_1>0$ are some finite constants. 
\end{definition}
\begin{definition}
    A function $\psi^*$ called the conjugate function, given $\psi$, is defined as $\psi^*(t) = \sup _{a\in \mathbb{R}}\{ta-\psi(a)\}$.
\end{definition}

\noindent\textbf{Assumptions.}
We use constants that hold uniformly over the data universe, including
all neighboring datasets. The coefficient $\lambda_0>0$ in
\eqref{penalized} is fixed throughout the analysis of DP Double-SPIDER.
\begin{assumption}\label{lipschitz}
For every sample $\xi$, the function $\ell(\cdot;\xi)$ is twice
continuously differentiable, with
$\|\nabla\ell(x;\xi)\|\leq G$ and
$\|\nabla^2\ell(x;\xi)\|\leq L_\ell$ for every $x\in\mathbb R^d$,
where $G,L_\ell>0$ are public constants.
\end{assumption}
\begin{assumption}[Local regularity of the conjugate]\label{smooth}
The generator $\psi$ is proper, closed, convex and nonnegative, satisfies
$\psi(1)=0$, and is $+\infty$ on $(-\infty,0)$.
Writing $\phi=\psi^*$, assume that $\phi$ is differentiable on an open
interval containing $I=[-U/\lambda_0,U/\lambda_0]$, where $U$ is the
loss bound in Assumption~\ref{bounded-loss}. On $I$,
\[
 \phi'(0)=1,\qquad 0\leq\phi'(s)\leq C,\qquad
 |\phi'(s)-\phi'(t)|\leq M|s-t|,
\]
for public constants $C\geq1$ and $M>0$.
This assumption is used for DP Double-SPIDER. It does not require a bounded
effective domain or nonnegativity of $\phi$.
\end{assumption}
\begin{assumption}[Bounded loss]\label{bounded-loss}\label{non-negative}
For every possible record and every $x\in\mathbb R^d$,
$0\leq\ell(x;\xi)\leq U$, where $U>0$ is public.
Initial points, clipping bounds, batch sizes, iteration counts and privacy
parameters are public and do not use unaccounted private data.
\end{assumption}

\section{DP Recursive-SPIDER}
\begin{algorithm}[!t]
\caption{DP Double-SPIDER}
\begin{algorithmic}[1]\label{dspider}
\STATE \textbf{Input:} initialization $(x_0, \eta_0)$, step sizes $\alpha_t, \beta_t$, epoch size $q$, number of iterations $T$, batch sizes $N_1, N_2, N_3, N_4$, Dataset $S$, Clipping thresholds $\{C_i\}_{i=1}^4$
\STATE Set $t=0$, choose $x_0$ and set $\eta_0=U/G$; use the settings in \eqref{f17:ds-schedule}--\eqref{f17:ds-noise}.
\WHILE{$t \leq T - 1$}
    \IF{$t \bmod q == 0$}
        \STATE Draw $N_1$  samples $\mathcal{B}_1$ from $S$ and compute $g_t \leftarrow \textbf{Clip}(\nabla_\eta \mathcal{L}(x_t, \eta_t;\mathcal{B}_1) ,C_1)+\omega_t$ where $\omega _t \sim \mathcal{N}(0,\sigma_1^2)$.
    \ELSE
        \STATE Draw $N_2$  samples $\mathcal{B}_2$ from $S$ and compute $g_t \leftarrow \textbf{Clip}(\nabla_\eta \mathcal{L}(x_t, \eta_t;\mathcal{B}_2) - \nabla_\eta \mathcal{L}(x_{t-1}, \eta_{t-1};\mathcal{B}_2),C_{2,t}) + g_{t-1}+\xi_t$, where $\xi _t \sim \mathcal{N}(0,\sigma_{2,t}^2)$.
    \ENDIF
    \STATE $\eta_{t+1} \leftarrow \Pi_{[0,U/G]}(\eta_t - \alpha_t g_t)$
    \IF{$t \bmod q == 0$}
        \STATE Draw $N_3$ samples $\mathcal{B}_3$ from $S$ and compute $v_t \leftarrow \textbf{Clip}(\nabla_x \mathcal{L}(x_t, \eta_{t+1};\mathcal{B}_3),C_3)+\tau_t$ where $\tau _t \sim \mathcal{N}(0,\sigma_3^2\mathbb{I}_d)$.  
    \ELSE
        \STATE Draw $N_4$  samples $\mathcal{B}_4$ from $S$ and compute $v_t \leftarrow \textbf{Clip}(\nabla_x \mathcal{L}(x_t, \eta_{t+1};\mathcal{B}_4) - \nabla_x \mathcal{L}(x_{t-1}, \eta_t;\mathcal{B}_4),C_{4,t}) + v_{t-1}+\chi_t$ where $\chi _t \sim \mathcal{N}(0,\sigma_{4,t}^2\mathbb{I}_d)$.
    \ENDIF
    \STATE $x_{t+1} \leftarrow x_t - \beta_t v_t$
    \STATE $t \leftarrow t + 1$
\ENDWHILE

\STATE \textbf{Return} $(x_\tau,\eta_\tau)$, where $\tau$ is independently uniform on $\{1,\ldots,T\}$.
\end{algorithmic}
\end{algorithm}
Recall that for DP-ERM with non-convex loss, the best-known result for utility is $\mathcal{O}((\frac{\sqrt{d}}{n \varepsilon})^{2/3})$ \citep{arora2023faster}. Thus, from Theorem \ref{main:1}, we can see that Algorithm \ref{dspider} achieves the same result only if $n\leq O(d^2)$. Thus, a natural question is  whether we can further improve the bound for DP-DRO? In this section, we will provide an affirmative answer for DRO with KL-divergence. 

For DP-DRO with KL-divergence, besides the dual problem in \eqref{PM}, it has been shown that the problem can be reformulated as the following compositional finite-sum optimization problem \citep{qi2022stochastic} with some $\rho$, which will be the main focus of this section:
\begin{equation}\label{compositional}
    \min_{x} \min_{\lambda\geq \lambda_0 } \Psi(x, \lambda) 
    := \lambda \log ( \frac{1}{n} \sum_{i=1}^{n} \exp ( \frac{\ell(x;\xi_i)}{\lambda} ) ) + \lambda  \rho.
\end{equation}
Specifically, by considering $\mathbf{w} = (x^T, \lambda)^T \in \mathbb{R}^{d+1}$ as a single variable to be optimized, the objective function is a compositional function of $\mathbf{w}$ in the form of $f(g(\mathbf{w}))$, where $g(\mathbf{w}) = [g_1, g_2]=\left[\lambda, \frac{1}{n}\sum_{i=1}^n \exp\left(\frac{\ell(x;\xi_i)}{\lambda}\right)\right] \in \mathbb{R}^2$ and $f(g) = g_1 \log(g_2) + g_1 \rho$. For convenience, we define $g(\mathbf{w};\xi) = \exp(\frac{\ell(x;\xi)}{\lambda})$. The gradient of $\Psi(\mathbf{w})$ is given by:
\begin{equation}\nonumber
    \begin{aligned}
        &\nabla _x \Psi (\mathbf{w}) = \nabla f(g(\mathbf{w}))\nabla _{x} g(\mathbf{w})\\
        &\nabla _{\lambda }\Psi (\mathbf{w}) = \nabla f (g(\mathbf{w})) \nabla _{\lambda}g(\mathbf{w})+ \log_{} (g(\mathbf{w}))+ \rho.
    \end{aligned}
\end{equation}
Note that since $f$ is independent of the data. Thus, it is sufficient to develop a private estimator for $g(\mathbf{w})$, $\nabla _{x} g(\mathbf{w})$, and $\nabla _{\lambda}g(\mathbf{w})$. Similar to problem \eqref{PM}, we can see that $x$ and $\lambda$ play different roles. Thus, when considering DP, we need to treat them separately, which motivates our DP Recursive-Spider (see Algorithm \ref{RSDRO} for details). DP Recursive-Spider adopts a similar idea as in DP Double-Spider, where we use a private variance-reduced gradient estimator for both $\nabla _{x} g(\mathbf{w})$ and $\nabla _{\lambda}g(\mathbf{w})$ in each iteration.  The algorithm proceeds in phases of length \( q \), during which it alternately updates \( x \) and \( \eta \) in each iteration. At the beginning of each phase (Lines 7--8), the gradient estimators are computed via empirical mean and subsequently privatized with Gaussian noise. In Lines 13--14, the estimators \( \mathbf{u}_t \) and \( v_t \) are updated via gradient variation for the subsequent \( q-1 \) iterations. After obtaining the gradient estimate of \( g(\mathbf{w}) \), the function value \( g(\mathbf{w}) \) is recursively estimated in Line 16. This is motivated by STORM~\citep{cutkosky2019momentum}, which could provide a better estimate of \( g(\mathbf{w})\). It is noteworthy that as \( \beta_t \to 0 \), \( s_t \) approaches the gradient variation, whereas when \( \beta_t \to 1 \), \( s_t \) behaves like the empirical mean.

The following  results provide privacy and utility guarantees of DP Recursive-Spider. For the full versions and their proofs, we refer readers to Appendix~\ref{proof}. 
 


\begin{theorem}\label{thm:3}\label{thm:RSDRO}
For KL divergence, under the conditions in Appendix~\ref{f17:conditions},
DP Recursive-SPIDER with $\lambda_t\equiv\lambda_0$ and the settings in
Appendix~\ref{f17:kl-settings} ($A<1$) is $(\varepsilon,\delta)$-DP and satisfies
\[
 \mathbb E\|\nabla\mathcal F(x_\tau)\|
 \leq\widetilde{\mathcal{O}}\!\left(
 \left(\frac{\sqrt{d\log(1/\delta)}}
 {n\varepsilon}\right)^{2/3}\right).
\]
\end{theorem}



\section{DP Double-SPIDER}


In this section, we will consider the general $\psi$-divergence. Due to the divergence regularization and the constrained simplex, it is hard to optimize efficiently. Thus, a popular approach is to investigate its dual formulation. By strong duality \citep{levy2020large,shapiro2017distributionally}, we can reformulate the objective function \eqref{penalized}  as
\begin{equation}\label{PM}
\min_{x} \Psi(x) = \inf_{\eta \in \mathbb{R}} \hat{\mathcal{L}}(x,\eta) 
:= \frac{\lambda }{n}\sum_{i=1}^n  \psi^*( \frac{\ell(x;\xi_i) - \eta}{\lambda} ) + \eta, 
\end{equation}
where $\eta \in \mathbb{R }$ is a dual variable. To simplify the discussion below, we  define $\mathcal{L}(x,\eta)=\hat{\mathcal{L}}(x,G\eta)$. Note that there are different forms of the dual problem; each will lead to different methods. The form of duality adopted in this work differs from the classical dual-based approach \citep{rafique2022weakly}. 

Under Assumptions \ref{lipschitz} and \ref{smooth}, the function $\Psi(x)$ is differentiable. With this reformulation,  \cite{jin2021non} showed that $\|\nabla_{x,\eta}\mathcal{L}(x,\eta)\| \leqslant \alpha/\sqrt{2}$ directly implies $\|\nabla \mathcal{F}(x)\| \leqslant \alpha$. Thus, our goal is to find private $x, \eta$ to make $\|\nabla_{x,\eta}\mathcal{L}(x,\eta)\| $ as small as possible. 

Note that problem \eqref{PM} is a finite-sum minimization problem. Specifically, let $z = (x, \eta)$ and $\mathcal{L}(x,\eta;\mathcal{S}):=\frac{1}{|\mathcal{S}|}\sum_{i=1}^{|\mathcal{S}|}\mathcal{L}(x,\eta;\xi_i)$ with $\mathcal{L}(x,\eta;\xi_i)=\lambda \psi^*(\frac{\ell(x;\xi_i)-\eta}{\lambda}) + \eta$. Thus, one direct way is directly adopting the previous method for DP-ERM, such as DP-SGD~\citep{bassily2014private,abadi2016deep} on the single composite variable $z = (x, \eta)$. In particular, we may use DP-SPIDER~\citep{arora2023faster}, which achieves the best-known rate for DP-ERM with non-convex loss. 

However, due to the specific structure of $\Psi$ in \eqref{PM}, directly using these methods may lead to sub-optimal performance. The main reason is that, under Assumptions~\ref{lipschitz} and \ref{smooth}, \citep{zhang2025revisiting} showed that $\mathcal{L}(x,\eta;\xi_i)$ is $(L_0, L_1)$-smooth in $x$ but $L_2$-smooth in $\eta$, with 
$L_0 = G + \frac{G^2 M}{\lambda}$, $L_1 = L/G$ and $L_2=\frac{G^2 M}{\lambda}$. Due to different smooth properties, training $x$ and $\eta$ uniformly will lead to worse noise added in each iteration.  In particular, for a fixed $x$, the function $\mathcal{L}(x, \cdot)$ is smooth, which facilitates a more accurate estimation of the gradient $\nabla_\eta \mathcal{L}$. Treating $z$ as a single variable, however, fails to exploit this property. By instead handling the parameter $x$ and the dual variable $\eta$ separately, we can more precisely characterize and control the behavior of the gradient estimators.

Motivated by this, we proposed our DP Double-SPIDER
(see Algorithm~\ref{dspider} for details), a private extension of SPIDER for DRO~\citep{zhang2025revisiting}, which is a variance-reduction optimization method based on SPIDER~\cite{fang2018spider} that achieves state-of-the-art (SOTA) convergence in the non-private case. Specifically, each variable is updated using a SPIDER-style variance-reduced gradient estimator perturbed with Gaussian noise. Concretely, in Line 4, for \(\eta\), we perform a large-batch refresh every \(q\) iterations with batch size \(N_1\) to form a low-variance estimate, and in Line 6 for the remaining \(q-1\) steps, we use incremental corrections with mini-batch size \(N_2\) via gradient differences, adding an independent Gaussian noise term at each step for privacy. After updating \(\eta\), we repeat the same refresh--increment pattern for \(x\), using batch sizes \(N_3\) and \(N_4\), again injecting calibrated Gaussian noise. This alternating procedure---refresh once, then correct \(q-1\) times---controls stochastic variance while enforcing privacy, yielding coordinated updates of \(\eta\) and \(x\) every \(q\) iterations.

Note that there are two differences compared with the non-private one: (1) As we mentioned, due to the different smooth properties of $x$ and $\eta$, we need the variance-reduced gradient estimator for $\nabla_\eta \mathcal{L}$ and $\nabla_x \mathcal{L}$ individually, i.e., it is a double version of SPIDER. (2) Under Assumptions \ref{lipschitz} and \ref{smooth}, we can show that the sensitivity of the gradient difference in the variance-reduced gradient estimator, such as  $\nabla_x \mathcal{L}(x_t, \eta_{t+1};\mathcal{B}) - \nabla_x \mathcal{L}(x_{t-1}, \eta_t;\mathcal{B})$, is $O(\max\{\|x_t-x_{t-1}\|, \|\eta_{t+1}-\eta_t\|\})$. Thus, the Gaussian noise added to ensure DP also depends on such differences, making the analysis complicated. 

The following results provide privacy and utility guarantees of DP Double-SPIDER. For the full versions and their proofs, we refer readers to Appendix \ref{proof}.

\begin{theorem}\label{thm:1}\label{main:1}
Under the conditions in Appendices~\ref{f17:conditions}
and~\ref{f17:ds-conditions}, DP Double-SPIDER with the settings in
\eqref{f17:ds-schedule} ($A_D<1$) is $(\varepsilon,\delta)$-DP and satisfies
\[
 \mathbb E\|\nabla\mathcal F(x_\tau)\|
 \leq \widetilde{\mathcal{O}} \!\left(
 \left(\frac{\sqrt{d\log(1/\delta)}}
 {n\varepsilon}\right)^{2/3}\right).
\]
\end{theorem}

\begin{remark}
Bounded loss and the projection $\eta\in[0,U/G]$ restrict the
arguments of $\psi^*$ to an interval where the sample gradients
are bounded and Lipschitz. Gradient differences consequently
have displacement dependent sensitivity, allowing their
estimation errors to be absorbed into descent.
Writing $A_D=\sqrt{(d+1)H_D}/(n\varepsilon)$ as in
\eqref{f17:ds-schedule}, the choices $T=q^2$ and
$\chi_B\leq q^{-2}$ yield
\begin{equation}
 \mathbb E\|\nabla\mathcal F(x_\tau)\|^2
 \leq \mathcal O\!\left(q^{-2}+A_D^2q\right).
\end{equation}
Balancing these terms with $q=\lceil A_D^{-2/3}\rceil$ gives
$\mathbb E\|\nabla\mathcal F(x_\tau)\|
=\mathcal O(A_D^{2/3})$.
Thus, under the stated assumptions, the divergence affects
problem-dependent constants while preserving the polynomial
utility rate attained for KL divergence.
\end{remark}

\begin{remark}
The KL structure analytically eliminates the scalar dual
variable and gives
\begin{equation}
 \mathcal F(x)=\lambda_0\log g(x,\lambda_0),
 \qquad
 \nabla\mathcal F(x)=
 \frac{\lambda_0\nabla_xg(x,\lambda_0)}{g(x,\lambda_0)}.
\end{equation}
Moreover, $\|\nabla^2\mathcal F(x)\|
\leq L_\ell+G^2/\lambda_0$.
DP Recursive-SPIDER therefore estimates the primal gradient
directly, avoiding iterative dual optimization and the
associated stationarity transfer.
After a privatized batch initialization, its normalizer
estimator uses one fresh sample per STORM update, while
the gradient trackers retain their batch refreshes.
These features provide a direct primal formulation and
constant sample cost for each normalizer update.
They motivate the KL specialization even though the two
methods have the same asymptotic utility rate; their total
computational costs are not ordered by this rate alone.
\end{remark}



\section{Experiments}\label{experiment}
\subsection{Experimental Setup}
\textbf{Dataset}
We consider 4  datasets that are widely used for studies in DRO: CIFAR10-ST  \citep{qi2022stochastic}, MNIST-ST \citep{zhang2025improved}, CelebA \citep{sagawa2019distributionally} and Fashion-MNIST. For constructing CIFAR10-ST, we artificially create imbalanced training data, where we only keep the last 100 images of each class for the first half of the classes, while keeping other classes and the test data unchanged. For constructing MNIST-ST, we convert MNIST into a binary task by mapping digits 0-4 $\to $ to class 0 and 5-9 $\to $ to class 1, and artificially create imbalanced training data by subsampling the positive class to a target imratio  $\in (0, 0.5]$ (fraction of positives); the test split is kept balanced (imratio = 0.5), and training batches are sampled with the same ratio to match the desired imbalance. 
 

\noindent\textbf{Baselines}
For the comparison of test accuracy, we compare different algorithms for optimizing the minimax problem \eqref{penalized}, including DP-SGDA \citep{rafique2022weakly} and Private Diff \citep{zhang2025improved}, across different privacy budgets $\varepsilon$ on the three aforementioned datasets. Besides, we take SCDRO \citep{zhang2023stochastic} as the non-private baseline, which is state-of-the-art.

\noindent\textbf{Evaluation Metrics} The evaluation of the selected four algorithms utilizes four primary metrics to assess utility, privacy, and optimization stability. The gradient norm serves as the central metric for measuring utility in terms of convergence to first-order stationary points. Predictive performance is evaluated via test accuracy across varying privacy budgets. To provide a more nuanced view of classifier performance, the Area Under the Curve (ROC AUC) and the F1 score are employed, with the latter specifically highlighting the balance between precision and recall in imbalanced scenarios. Finally, the Membership Inference Attack (MIA) AUC is used as an empirical measure of privacy robustness, where lower values signify superior protection against training data leakage.


\textbf{MIA Experiments}
We evaluate the privacy robustness of different differentially private optimization algorithms using membership inference attacks (MIA) on CIFAR10-ST. Experiments are conducted under privacy budgets $\varepsilon \in \{0.1, 1, 3, 5, 8, 10\}$, covering both high-privacy and low-privacy regimes. For each algorithm–privacy budget configuration, we perform five independent runs with different random seeds and report results as mean $\pm$ standard deviation. We focus on two DP DRO methods: DP Double SPIDER and DP Recursive SPIDER. Privacy leakage is evaluated using a standard black-box MIA, where the attacker infers whether a data point was part of the training set based solely on the model’s output confidence. Attack performance is measured by the area under the ROC curve (AUC).

\subsection{Experimental Results}
\begin{figure*}[!t]
    \centering
    \includegraphics[width=0.49\linewidth]{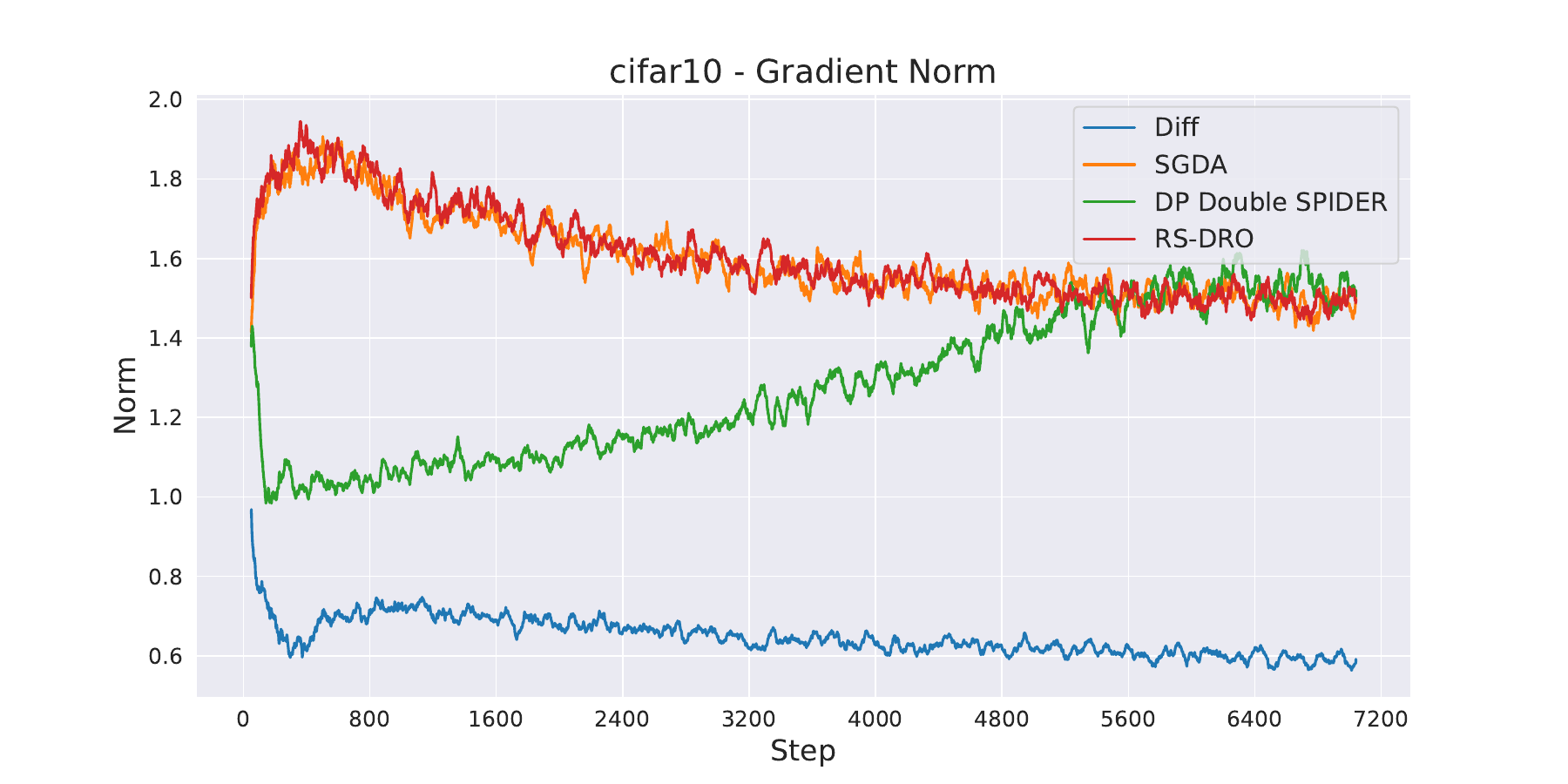}
    \hfill
    \includegraphics[width=0.49\linewidth]{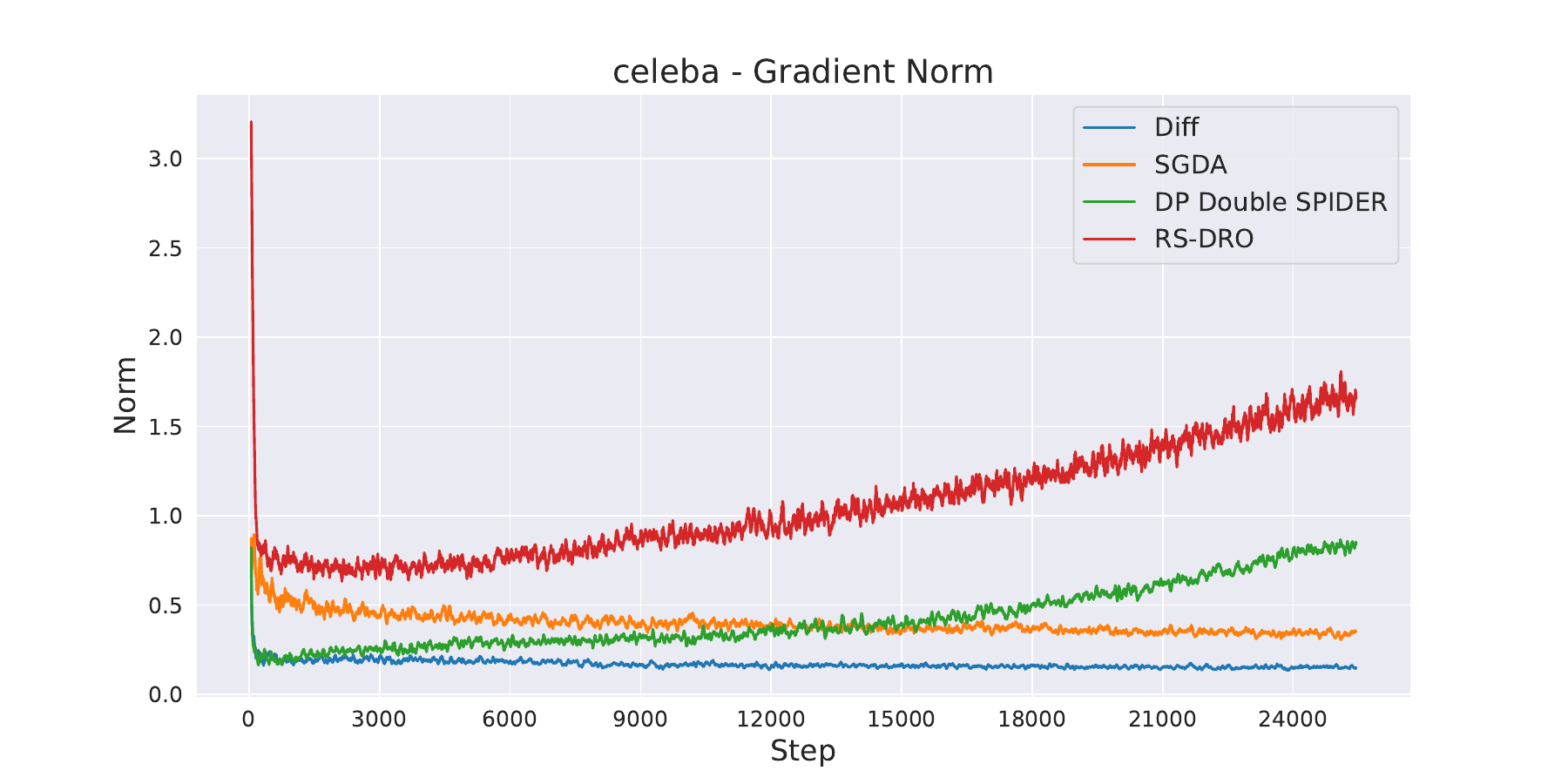}
    
    \vspace{0.5cm}
    
    \includegraphics[width=0.49\linewidth]{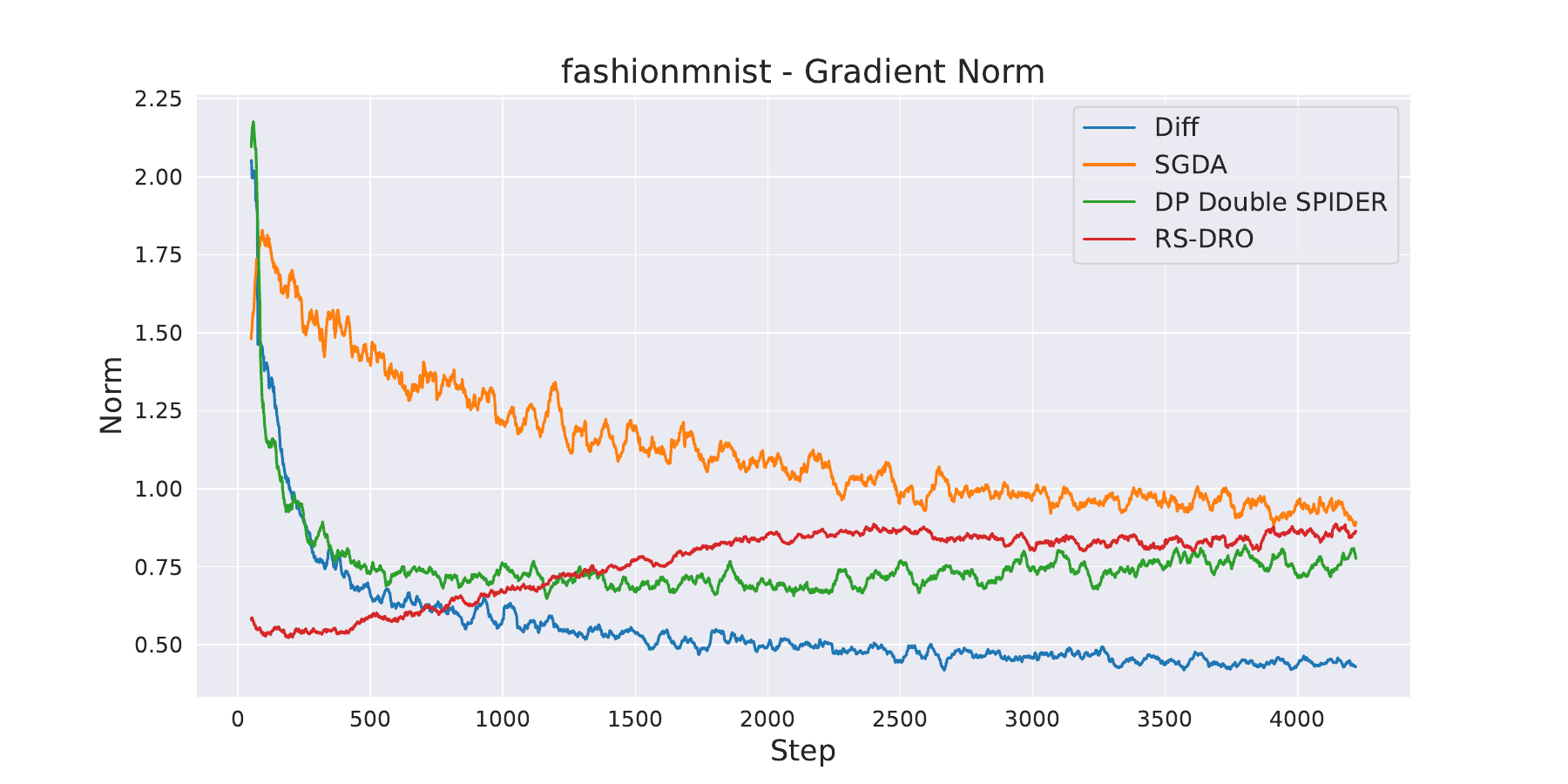}
    \hfill
    \includegraphics[width=0.49\linewidth]{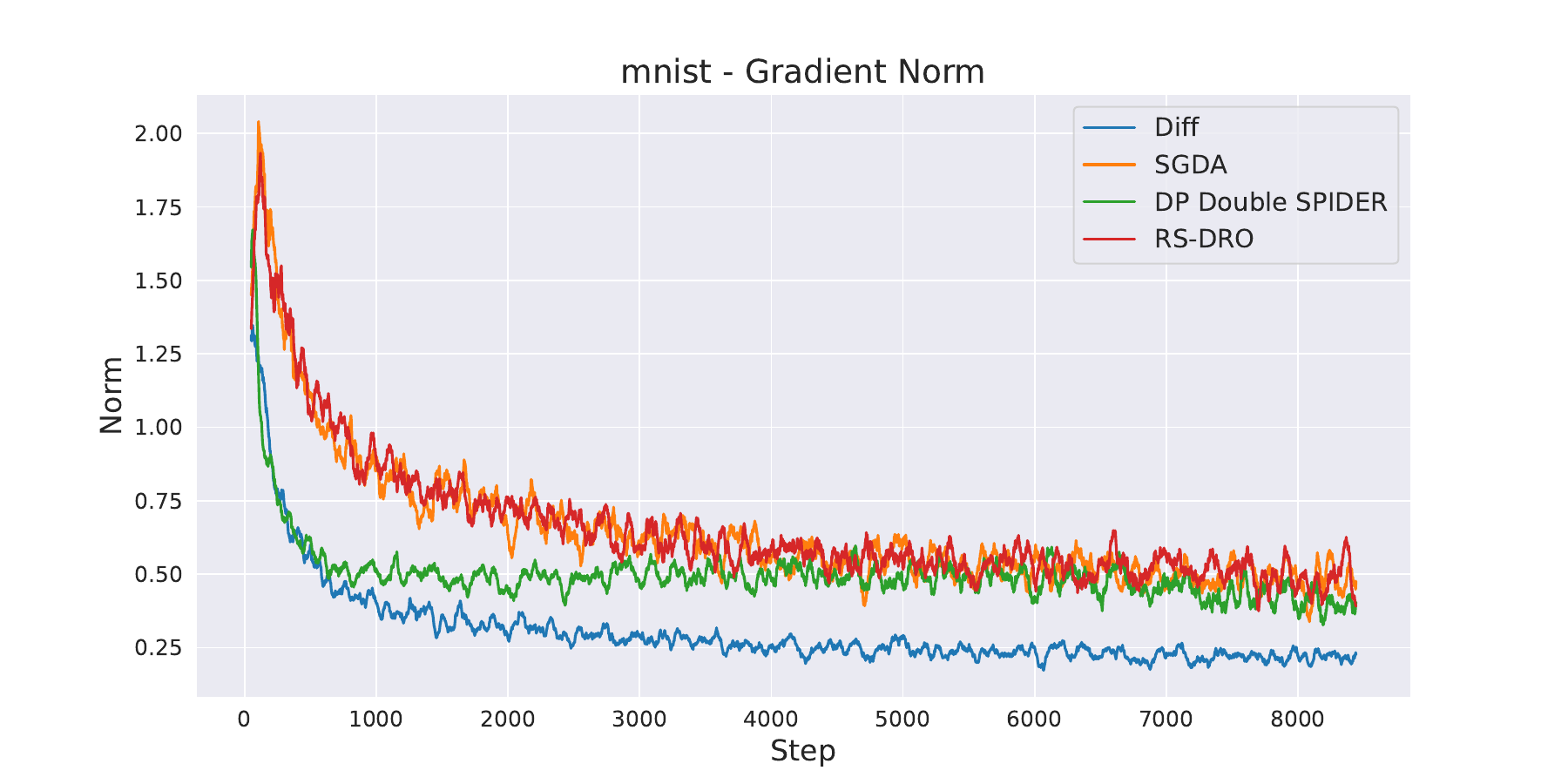}
    
    \caption{Experimental Results: The performances of four algorithms on CIFAR10-ST, CelebA, Fashion-MNIST, MNIST-ST respectively.}
    \label{fig:combined}
\end{figure*}

We conducted experiments on all four datasets and present representative results for three of them here. Additional experiments on AUC, F1 scores, and hyperparameter sensitivity  are provided in the Appendix \ref{auc}, \ref{f1}, and \ref{parameter}, respectively.

\textbf{Performance Analysis.} Table \ref{exp_res} illustrates that our proposed algorithms, DP Double-SPIDER (DP DS) and DP Recursive-SPIDER (RS DRO), consistently outperform the baseline methods (DP-SGDA and PrivateDiff) across nearly all datasets and privacy budget settings. The only exception is the MNIST-ST dataset, where all four algorithms achieve high accuracy ($\approx 99\%$), reaching a performance plateau that makes them effectively indistinguishable. 

Due to the absence of base model fine tuning, the absolute accuracy on CIFAR10-ST is lower than on other benchmarks. However, the performance gap between our methods and the baselines is most pronounced on this dataset, particularly under stringent privacy constraints (e.g., $\varepsilon = 0.5$).

We draw several key observations from these results:

\begin{itemize}
    \item \textbf{Performance Across Datasets:} On CIFAR10-ST, CelebA, and Fashion-MNIST, our proposed algorithms maintain a clear and significant lead over the baselines. This suggests that our variance-reduction approach is more effective for complex data distributions where gradient noise typically hinders convergence.
    
    \item \textbf{Impact of Privacy Budget $\varepsilon$:} As expected, the test accuracy for all algorithms improves as the privacy budget $\varepsilon$ increases. Notably, our algorithms exhibit superior stability; they provide high utility even with very small $\varepsilon$, whereas the baseline methods suffer significantly more from the increased noise required for higher privacy levels.
    
    \item \textbf{Sensitivity Comparison:} The relative improvement in accuracy as $\varepsilon$ increases from 0.5 to 10 is more substantial for DP-SGDA and PrivateDiff, especially on CIFAR10-ST. This high sensitivity indicates that these baseline methods are less robust to the noise injection process, while our methods achieve near-optimal utility even with limited privacy budgets.
\end{itemize}

\begin{table*}[htbp]
\centering
\caption{Membership Inference Attack (MIA) results under different privacy budgets. Lower AUC indicates better privacy. Results are reported as mean $\pm$ standard deviation over 5 runs. Bold values indicate best (lowest) privacy leakage for each privacy budget. DP Recursive SPIDER consistently achieves lower MIA AUC compared to DP Double SPIDER across all epsilon values.}
\label{tab:mia_results}
\setlength{\tabcolsep}{3pt}
\resizebox{\textwidth}{!}{%
\begin{tabular}{lcccccc}
\toprule
\multirow{2}{*}{Algorithm}
& \multicolumn{6}{c}{MIA AUC (Mean $\pm$ Std)} \\
\cmidrule(lr){2-7}
& $\varepsilon{=}0.1$ & $\varepsilon{=}1$ & $\varepsilon{=}3$ & $\varepsilon{=}5$ & $\varepsilon{=}8$ & $\varepsilon{=}10$ \\
\midrule
DP Double SPIDER
& $0.9719\pm0.0059$
& $0.9723\pm0.0045$
& $0.9717\pm0.0025$
& $0.9723\pm0.0021$
& $0.9724\pm0.0033$
& $0.9715\pm0.0022$ \\

DP Recursive SPIDER
& $\mathbf{0.8234\pm0.0127}$
& $\mathbf{0.7932\pm0.0072}$
& $\mathbf{0.7958\pm0.0240}$
& $\mathbf{0.8014\pm0.0383}$
& $\mathbf{0.8001\pm0.0302}$
& $\mathbf{0.7825\pm0.0149}$ \\

\bottomrule
\end{tabular}%
}
\end{table*}

\begin{table*}[tbp]
\centering
\caption{Comparison of Test Accuracy (\%) Among Two Baseline Methods (DP SGDA and PrivateDiff Minimax) and Our Algorithms (DP Double-SPIDER and DP Recursive-SPIDER).  SGDA: DP-SGDA; Private Diff: PrivateDiff Minimax; DP DS: DP Double SPIDER; RS DRO: DP Recursive SPIDER. The results for non-private are give by SCDRO \citep{zhang2023stochastic}. }
\label{exp_res}
\setlength{\tabcolsep}{3pt}
\setlength{\tabcolsep}{2pt}  
\resizebox{\textwidth}{!}{%
\begin{tabular}{lcccccccccccccccc}
\toprule
\multirow{2}{*}{Dataset} & \multicolumn{4}{c}{CIFAR10-ST} & \multicolumn{4}{c}{MNIST-ST} & \multicolumn{4}{c}{Fashion-MNIST} & \multicolumn{4}{c}{CelebA} \\
\cmidrule(lr){2-5} \cmidrule(lr){6-9} \cmidrule(lr){10-13} \cmidrule(lr){14-17}
& SGDA & Private Diff & DP DS & RS DRO & SGDA & Private Diff & DP DS & RS DRO & SGDA & Private Diff & DP DS & RS DRO & SGDA & Private Diff & DP DS & RS DRO \\
\midrule
Non-private & - & - & - & 67.58 & - & - & - & - & - & - & - & - & - & - & - & - \\
$\varepsilon = 0.5$ & 41.80 & 43.06 & \textbf{56.31} & 53.81 & 99.16 & 99.18 & 99.49 & \textbf{99.66} & 88.13 & 88.82 & 91.39 & \textbf{92.27} & 88.16 & 88.18 & 89.37 & \textbf{90.37} \\
$\varepsilon = 1$ & 45.10 & 46.13 & 55.67 & \textbf{55.93} & 99.48 & 99.30 & 99.50 & \textbf{99.56} & 89.16 & 89.51 & 91.37 & \textbf{91.87} & 89.17 & 89.01 & 89.37 & \textbf{90.32} \\
$\varepsilon = 5$ & 51.54 & 51.29 & 55.69 & \textbf{56.58} & \textbf{99.49} & 99.47 & 99.47 & 99.45 & 90.45 & 90.70 & 91.36 & \textbf{91.59} & 89.76 & 90.14 & 89.38 & \textbf{90.41} \\
$\varepsilon = 10$ & 52.47 & 51.68 & \textbf{57.26} & 55.67 & \textbf{99.55} & 99.32 & 99.46 & 99.45 & 90.39 & 91.29 & 91.38 & \textbf{91.58} & 89.93 & 90.35 & 89.38 & \textbf{90.45} \\
\bottomrule
\end{tabular}
}%
\end{table*}
\begin{table*}[!t]
\centering
\caption{Clipping-threshold robustness of DP Double-SPIDER on CIFAR10-ST. All four clipping thresholds are tied as \(C_1=C_2=C_3=C_4=C\). We report test accuracy (\%) as mean \(\pm\) standard deviation over three random seeds.}
\label{tab:main_uniform_clip}
\resizebox{\textwidth}{!}{
\begin{tabular}{cccccccc}
\toprule
\(\varepsilon\) 
& \(C=0.1\) 
& \(C=0.5\) 
& \(C=1.0\) 
& \(C=2.0\) 
& \(C=5.0\) 
& \(C=10.0\) 
& Range \\
\midrule
\(1.0\)  
& \(45.80 \pm 0.11\) 
& \(45.87 \pm 0.04\) 
& \(45.87 \pm 0.05\) 
& \(45.87 \pm 0.04\) 
& \(45.87 \pm 0.04\) 
& \(45.87 \pm 0.04\) 
& \(0.07\) \\

\(5.0\)  
& \(46.02 \pm 0.20\) 
& \(46.05 \pm 0.17\) 
& \(46.05 \pm 0.24\) 
& \(46.07 \pm 0.24\) 
& \(46.08 \pm 0.21\) 
& \(46.10 \pm 0.21\) 
& \(0.08\) \\

\(10.0\) 
& \(45.99 \pm 0.12\) 
& \(46.02 \pm 0.06\) 
& \(46.05 \pm 0.12\) 
& \(46.10 \pm 0.15\) 
& \(46.12 \pm 0.14\) 
& \(46.12 \pm 0.12\) 
& \(0.13\) \\
\bottomrule
\end{tabular}
}
\end{table*}
\textbf{Gradient Norm Analysis.}
Figure~\ref{fig:combined} illustrates the evolution of gradient norms on CIFAR10-ST, CelebA, Fashion-MNIST, and MNIST-ST, comparing DP-SGD and SGDA with the proposed DP Double SPIDER and DP Recursive SPIDER methods. Across most datasets, the baseline methods exhibit relatively large gradient norms and noticeable oscillations throughout the training process. In particular, DP-SGDA often shows slow convergence with persistent variance, while SGDA displays pronounced fluctuations, especially in later iterations. These behaviors indicate that privacy-induced noise and stochastic saddle-point dynamics hinder stable convergence to stationary points.

In contrast, DP Double SPIDER  achieves faster decay and substantially lower variance in gradient norms, reflecting the effectiveness of variance reduction via gradient-difference estimators. DP Recursive SPIDER further enhances stability, maintaining uniformly small gradient norms with minimal oscillations, particularly in the later stages where baseline methods plateau or become unstable. Overall, these results demonstrate that our algorithms provide significantly more reliable and stable optimization dynamics under differential privacy, leading to improved convergence toward first-order stationary solutions in private nonconvex--strongly-concave optimization.

\textbf{Hyperparameter Sensitivity.}  A natural concern is that DP Double-SPIDER introduces four clipping thresholds, \(C_1,C_2,C_3,C_4\), which may appear to increase the tuning burden. We therefore conduct a controlled sensitivity study on CIFAR10-ST, using a fixed 20-epoch ResNet20 pretraining protocol and reporting the mean \(\pm\) standard deviation of test accuracy over three random seeds. To test whether the four thresholds need to be tuned independently, we tie them to a single value, \(C_1=C_2=C_3=C_4=C\), and sweep \(C\) over two orders of magnitude. As shown in Table~\ref{tab:main_uniform_clip}, the accuracy is nearly unchanged across \(C\in\{0.1,0.5,1,2,5,10\}\). The range of mean accuracy is only \(0.07\), \(0.08\), and \(0.13\) percentage points for \(\varepsilon=1,5,10\), respectively. These results indicate that the four clipping thresholds in DP Double-SPIDER can be tied to a single default value in practice. Thus, although the algorithm contains multiple clipping constants in its theoretical statement, they do not create a practical fine-tuning burden.

Additional sensitivity studies, including pairwise clipping, random independent clipping, comparison with baselines, and DP Recursive-SPIDER hyperparameters, are reported in Appendix~\ref{parameter}.

\textbf{MIA Experiments.} The results of the Membership Inference Attack (MIA) are summarized in Table~\ref{tab:mia_results}. The empirical results demonstrates that DP Recursive SPIDER consistently outperforms DP Double SPIDER, achieving significantly lower AUC values across all evaluated privacy budgets $\varepsilon$. While DP Double SPIDER exhibits high and nearly static privacy leakage (AUC $\approx 0.97$) regardless of the privacy budget, DP Recursive SPIDER provides substantially better protection, with AUC values ranging between $0.78$ and $0.82$. Furthermore, DP Recursive SPIDER shows greater sensitivity to variations in $\varepsilon$, whereas DP Double SPIDER remains largely unaffected by changes in the privacy budget within the tested range. Overall, the superior privacy-preserving performance of DP Recursive SPIDER is evident across all settings, which aligns with our theoretical analysis.

\section{Conclusion}
In this paper, we investigate the Distributionally Robust Optimization (DRO) problem with differential privacy and non-convex loss functions. Specifically, we first analyze the Double SPIDER DRO, where the objective function is reformulated as a minimization problem via primal-dual transformation. Then,  with the KL-divergence, we recast the original objective as a compositional problem;  the inner component is a finite-sum loss function, a structure that facilitates the application of efficient optimization techniques. Inspired by STORM, we introduce the Recursive Double SPIDER, which achieves improved error bounds and matches the best-known results for  DP minimax optimization with non-convex loss. Finally, we evaluate our methods against baselines on four datasets, and the empirical results demonstrate superior performance and corroborate our theoretical analysis.

\subsection*{AI use statement}

In this work, we used generative AI tools for  polishing the paper writing.
We have not used generative AI tools for proving the theorems.
Additionally, we used generative AI tools for the code of the experiments. We have reviewed all AI-assisted work. We have checked the LLM-generated code and review the polished paper by the authors.

%% file: sec/3_finalcopy.tex



%% file: author-kit-CVPR2027-v1-latex/sec/appendix.tex
\clearpage
\onecolumn
\raggedbottom
\setcounter{page}{1}
\renewcommand{\thepage}{S\arabic{page}}
\begin{center}
  {\Large\bfseries\thetitle\par}
  \vspace{0.5em}
  {\large Supplementary Material\par}
\end{center}

\section{Algorithms}\label{app:algorithms}
Both algorithms use the fixed penalty $\lambda_0$ in $\mathcal F$ and the sampling convention in Appendix~\ref{f17:conditions}.

\begin{algorithm}[htbp]
\caption{Clipping $(x, C)$}
\begin{algorithmic}[1]
\REQUIRE $x$ and clipping threshold $C\geq0$.
\STATE $\hat x=x$ if $\|x\|_2\leq C$; otherwise $\hat x=Cx/\|x\|_2$.
\RETURN $\hat{x}$.
\end{algorithmic}
\end{algorithm}

\begin{algorithm}
\caption{DP Recursive-SPIDER}
\begin{algorithmic}[1]\label{RSDRO}
\REQUIRE Dataset: $S \in \mathcal{X}^n$, Function: $f : \mathbb{R}^d \times \mathcal{X} \to \mathbb{R}$, Learning Rate: $\eta$, Phase Size: $q$, Batch Sizes $\{b_i\}_{i=1}^4$, Privacy Parameters: $(\varepsilon, \delta)$, Iterations: $T$, Clipping threashold: $\{C_i\}_{i=1}^5$ and $\mathbf{w} = (x,\lambda)$.  

\STATE Choose $x_0$ and set $\mathbf w_0=(x_0,\lambda_0)$; keep $\lambda_t=\lambda_0$ and use Appendix~\ref{f17:kl-settings}.
\STATE Draw a fresh batch $\mathcal B_0$ of size $B$ and set $s_0=\Pi_{[1,E]}(B^{-1}\sum_{i\in\mathcal B_0}\textbf{Clip}(g(\mathbf w_0;\xi_i),E)+\zeta_0^{\rm init})$, where $\zeta_0^{\rm init}\sim\mathcal N(0,E^2\nu_6(1,B))$.

\FOR{$t = 0, \ldots, T-1$}
    
    \STATE Update $\mathbf v_t,u_t$ using the following fresh batches and noises:

    \IF{$\bmod(t, q) = 0$}
        \STATE Sample batch $\mathcal{B}_1$, $\mathcal{B}_3$ of size $b_1$, $b_3$ respectively.
        \STATE Sample $\nu_t \sim \mathcal{N}(0, \mathbf{I}_d\sigma_1^2)$, Sample $\omega _t \sim \mathcal{N}(0, \sigma_3^2)$ 
        \STATE $\mathbf{v}_{t} = \frac{1}{b_1} \sum_{\xi _i \in \mathcal{B}_1} \textbf{Clip}(\nabla_{x} g(\mathbf{w}_{t}; \xi _i),C_1) + \nu _t$
        \STATE $u_{t} = \frac{1}{b_3} \sum_{\xi_i \in \mathcal{B}_3} \textbf{Clip}(\nabla_\lambda  g(\mathbf{w}_t; \xi _i),C_3) + \omega _t$
    \ELSE
        \STATE Sample batch $\mathcal{B}_2$, $\mathcal{B}_4$ of size $b_2$, $b_4$ respectively. 
        \STATE Sample $\kappa_t\sim\mathcal N(0,\mathbf I_d\sigma_{2,t}^2)$
         \STATE Sample $\zeta_t\sim\mathcal N(0,\sigma_{4,t}^2)$
        \STATE $\mathbf{v}_t = \mathbf v_{t-1}+\frac{1}{b_2} \sum_{\xi _i \in \mathcal{B}_2} \textbf{Clip}([\nabla_x g(\mathbf{w}_t; \xi _i) - \nabla _xg(\mathbf{w}_{t-1}; \xi _i)],C_{2,t}) + \kappa_t$
        \STATE $u_t = u_{t-1}+\frac{1}{b_4} \sum_{\xi _i \in \mathcal{B}_4} \textbf{Clip}([\nabla_{\lambda} g(\mathbf{w}_t; \xi _i) - \nabla _{\lambda} g(\mathbf{w}_{t-1}; \xi _i)],C_{4,t}) + \zeta_t$
        
    \ENDIF 

    \STATE $\mathbf z_t=((\lambda_0/s_t)\mathbf v_t^{\top},(\lambda_0/s_t)u_t+\log(s_t)+\rho)^{\top}$.
    \STATE Update $\mathbf w_{t+1}=(x_t-\eta(\mathbf z_t)_x,\lambda_0)$.
    \STATE Draw one fresh uniform sample $\xi_0$ and independent $\zeta_t^s\sim\mathcal N(0,\sigma_{s,t}^2)$.
    \STATE $s_{t+1}=\Pi_{[1,E]}\!\left(\textbf{Clip}(g(\mathbf w_{t+1};\xi_0),E)+(1-\beta)(s_t-\textbf{Clip}(g(\mathbf w_t;\xi_0),E))+\zeta_t^s\right)$.

    
\ENDFOR

\STATE \textbf{return} $(\mathbf{w}_{\tau}, \mathbf{v}_{\tau}, u_{\tau}, s_{\tau})$, where $\tau$ is independently uniform on $\{0,\ldots,T-1\}$
\end{algorithmic}
\end{algorithm}

\section{Lemmas}
\subsection{Conditions and notation}\label{f17:conditions}
Let $n\geq8$, $0<\varepsilon\leq1$, $0<\delta<1$, and fix $\lambda_0>0$.
For every sample $\xi$ and $x\in\mathbb R^d$, assume that $\ell(\cdot;\xi)$ is twice continuously differentiable and, for $U,G,L_\ell>0$,
\[
 0\leq\ell(x;\xi)\leq U,\qquad
 \|\nabla\ell(x;\xi)\|\leq G,\qquad
 \|\nabla^2\ell(x;\xi)\|_{\rm op}\leq L_\ell.
\]
The bounds and initial points are independent of $S$. Each batch is a fresh uniform subset sampled without replacement, with its indices hidden. Conditional on all previous noisy quantities, the current batches and Gaussian noises are independent. Parameter choices below depend only on $n,d,\varepsilon,\delta$ and these bounds.
Write $\ell_i(x)=\ell(x;\xi_i)$ and
\begin{equation}\label{f17:primal}
 \mathcal F(x)=\max_{p\in\Delta_n}
 \left\{\sum_{i=1}^np_i\ell_i(x)-\frac{\lambda_0}{n}\sum_{i=1}^n\psi(np_i)\right\},
 \qquad \Delta_n=\{p\geq0:\textstyle\sum_i p_i=1\}.
\end{equation}
Expectations are over the algorithm's randomness with $S$ fixed.

\subsection{Conditions for the scalar dual}\label{f17:ds-conditions}
For DP Double-SPIDER, $\psi$ is proper, closed, convex and nonnegative, $\psi(1)=0$, and $\psi(a)=+\infty$ for $a<0$.
Let $\phi=\psi^*$ be differentiable on an open neighborhood of $I=[-U/\lambda_0,U/\lambda_0]$, with
\[
 \phi'(0)=1,\qquad 0\leq\phi'(s)\leq C,\qquad
 |\phi'(s)-\phi'(t)|\leq M|s-t|\quad(s,t\in I),\qquad C\geq1,\ M>0.
\]
Set
\begin{equation}\label{f17:ds-dual}
 \begin{aligned}
 h_i(x,\eta)&=\lambda_0\phi\!\left(\frac{\ell_i(x)-G\eta}{\lambda_0}\right)+G\eta,
 &\mathcal L(x,\eta)&=\frac1n\sum_{i=1}^n h_i(x,\eta),
 &\mathcal K_\eta&=[0,U/G],
 \end{aligned}
\end{equation}
and let $\mathcal L(x,\eta;\mathcal B)$ denote the corresponding batch average. Define
\begin{equation}\label{f17:ds-constants}
 B_\eta=G\max\{1,C-1\},\quad B_x=CG,\quad
 B_D^2=B_\eta^2+B_x^2,\quad L_D=CL_\ell+2MG^2/\lambda_0.
\end{equation}

\begin{lemma}\label{partial_smoothness}(Partial smoothness).
For any $x,x'\in\mathbb R^d$ and $\eta,\eta'\in\mathcal K_\eta$,
\begin{align*}
 \|\nabla_x\mathcal L(x,\eta)-\nabla_x\mathcal L(x',\eta)\|&\leq L_0\|x-x'\|,\\
 |\nabla_\eta\mathcal L(x,\eta)-\nabla_\eta\mathcal L(x,\eta')|&\leq L_2|\eta-\eta'|,
\end{align*}
where $L_0=CL_\ell+MG^2/\lambda_0$ and $L_2=MG^2/\lambda_0$. Each $\nabla h_i$ is $L_D$-Lipschitz on $\mathbb R^d\times\mathcal K_\eta$.
\end{lemma}
\begin{proof}
Let $z=(x,\eta)$, $z'=(x',\eta')$ and $a_i(z)=\phi'((\ell_i(x)-G\eta)/\lambda_0)$. Then
\begin{align*}
 \nabla h_i(z)&=(a_i(z)\nabla\ell_i(x),\,G(1-a_i(z))),\\
 |a_i(z)-a_i(z')|&\leq\frac{MG}{\lambda_0}(\|x-x'\|+|\eta-\eta'|)
 \leq\frac{\sqrt2MG}{\lambda_0}\|z-z'\|,\\
 \|\nabla h_i(z)-\nabla h_i(z')\|
 &\leq CL_\ell\|x-x'\|+\sqrt2G|a_i(z)-a_i(z')|
 \leq L_D\|z-z'\|.
\end{align*}
Holding either block fixed gives $L_0$ and $L_2$.
\end{proof}

\begin{lemma}\label{partial_noise}(Variance bounds).
For a uniform sample index $i$,
\[
 \mathbb E_i\|\nabla_xh_i-\nabla_x\mathcal L\|^2\leq D_0+D_1|\nabla_\eta\mathcal L|^2,
 \qquad \mathbb E_i|\nabla_\eta h_i-\nabla_\eta\mathcal L|^2\leq D_2,
\]
where $D_0=B_x^2$, $D_1=0$ and $D_2=B_\eta^2$.
For a fresh uniform $k$-subset $\mathcal B$ and vectors $a_1,\ldots,a_n$,
\begin{equation}\label{f17:finite-pop}
 \mathbb E\left\|\frac1k\sum_{i\in\mathcal B}a_i-\bar a\right\|^2
 =\chi_k\frac1n\sum_i\|a_i-\bar a\|^2,
 \quad \chi_k=\frac{n-k}{k(n-1)},\quad \bar a=\frac1n\sum_i a_i.
\end{equation}
\end{lemma}
\begin{proof}
The first two bounds follow from $\|\nabla_xh_i\|\leq B_x$ and $|\nabla_\eta h_i|\leq B_\eta$.
Writing $b_i=a_i-\bar a$, we have
\[
 \mathbb E\left\|\frac1k\sum_{i\in\mathcal B}b_i\right\|^2
 =\frac1{k^2}\left[\frac{k}{n}\sum_i\|b_i\|^2
 +\frac{k(k-1)}{n(n-1)}\sum_{i\ne j}\langle b_i,b_j\rangle\right]
 =\chi_k\frac1n\sum_i\|b_i\|^2.
\]
\end{proof}

To analyze Algorithm~\ref{dspider}, we further develop the following property of gradients $\nabla_x h_i$ and $\nabla_\eta h_i$.
\begin{lemma}\label{L_2 continuous}
For any $x,x'\in\mathbb R^d$, $\eta,\eta'\in\mathcal K_\eta$ and sample $i$,
\[
 |\nabla_\eta h_i(x,\eta)-\nabla_\eta h_i(x',\eta)|\leq L_2\|x-x'\|,
 \qquad \|\nabla_xh_i(x,\eta)-\nabla_xh_i(x,\eta')\|\leq L_2|\eta-\eta'|.
\]
The same inequalities hold for $\mathcal L$.
\end{lemma}

\begin{theorem}\label{gaussian}\label{thm:gaussian}\label{f17:privacy}
Consider $s$ interleaved adaptive query sequences, with call caps $1\leq K_j\leq n^2$ and batch sizes $1\leq k_j\leq n$ fixed independently of $S$.
Conditioned on the previous noisy quantities, suppose every sample contribution has norm at most $C_t$; the query and $C_t$ may depend only on those quantities and fixed parameters.
A fresh hidden uniform $k_j$-subset average with Gaussian coordinate variance $C_t^2\nu_s(K_j,k_j)$ is jointly $(\varepsilon,\delta)$-DP under replacement adjacency, where
\begin{equation}\label{f17:privacy-noise}
 \nu_s(K,k)=\frac{256s^2H_s}{\varepsilon^2}\max\{k^{-2},K/n^2\},
 \qquad H_s=\log\frac{4sn^2}{\delta}\log\frac{2s}{\delta}.
\end{equation}
When $C_t=0$, the contribution and its noise are zero.
\end{theorem}
\begin{proof}
Fix the same previous noisy quantities on neighboring datasets. Put
\[
 h=\log(2s/\delta),\quad \gamma=k/n,\quad
 \delta_0=\frac{\delta}{2sK},\quad
 \varepsilon_0=\min\left\{\frac12,\frac{\varepsilon}{4s\gamma\sqrt{2Kh}}\right\}.
\]
The sensitivity is $2C_t/k$. By the Gaussian mechanism and subsampling amplification~\cite{dwork2014algorithmic,balle2018privacy},
\begin{align*}
 \frac{2C_t\sqrt{2\log(4sK/\delta)}}{k\varepsilon_0}
 &=C_t\sqrt{\log(4sK/\delta)}
 \max\left\{\frac{4\sqrt2}{k},\frac{16s\sqrt{Kh}}{n\varepsilon}\right\}
 \leq C_t\sqrt{\nu_s(K,k)},\\
 \varepsilon'&=\log(1+\gamma(e^{\varepsilon_0}-1))
 \leq2\gamma\varepsilon_0\leq\frac{\varepsilon}{2s\sqrt{2Kh}}=:a_j,
 &\delta'&=\gamma\delta_0\leq\frac{\delta}{2sK}=:d_j.
\end{align*}
The privacy loss proof of adaptive composition~\cite[Theorem~3.20]{dwork2014algorithmic}, with these deterministic per-call bounds and $\delta_*=\delta/2$, gives in chronological order
\begin{equation}\label{f17:privacy-composition}
 \begin{aligned}
 \varepsilon_{\rm tot}
 &\leq\sqrt{2\log(1/\delta_*)\sum_{j=1}^sK_ja_j^2}
       +\sum_{j=1}^sK_ja_j(e^{a_j}-1)\\
 &\leq\frac{\varepsilon}{2\sqrt s}
       +\frac{\varepsilon^2}{4s\log(2s/\delta)}<\varepsilon,
 &\delta_{\rm tot}&\leq\delta_*+\sum_{j=1}^sK_jd_j=\delta.
 \end{aligned}
\end{equation}
Here $\sum_jK_ja_j^2=\varepsilon^2/(8sh)$ and $e^{a_j}-1\leq2a_j$. Each bound conditions on the complete previous history, including the other query sequences.
\end{proof}

\section{Omitted Proof}\label{proof}\label{f17:proof}
Besides Lemma~\ref{partial_smoothness}, we have the following depiction of $\nabla_\eta\mathcal L(x,\eta)$:
\begin{lemma}
For $x,x'\in\mathbb R^d$ and $\eta\in\mathcal K_\eta$,
\[
 |\nabla_\eta\mathcal L(x,\eta)-\nabla_\eta\mathcal L(x',\eta)|
 \leq\frac{MG^2}{\lambda_0}\|x-x'\|.
\]
\end{lemma}
\begin{proof}
Recall the definition of $\mathcal L(x,\eta)$. By the $M$-Lipschitz continuity of $\phi'$ and the $G$-Lipschitz continuity of $\ell_i$,
\begin{align*}
 &|\nabla_\eta\mathcal L(x,\eta)-\nabla_\eta\mathcal L(x',\eta)|\\
 &=\left|\frac G n\sum_{i=1}^n\left[
 \phi'\!\left(\frac{\ell_i(x)-G\eta}{\lambda_0}\right)
 -\phi'\!\left(\frac{\ell_i(x')-G\eta}{\lambda_0}\right)\right]\right|\\
 &\leq\frac G n\sum_{i=1}^n\frac M{\lambda_0}|\ell_i(x)-\ell_i(x')|
 \leq\frac{MG^2}{\lambda_0}\|x-x'\|.
\end{align*}
\end{proof}
Similarly, we have
\begin{lemma}
For $x\in\mathbb R^d$ and $\eta,\eta'\in\mathcal K_\eta$,
\[
 \|\nabla_x\mathcal L(x,\eta)-\nabla_x\mathcal L(x,\eta')\|
 \leq\frac{MG^2}{\lambda_0}|\eta-\eta'|.
\]
\end{lemma}
\begin{proof}
\begin{align*}
 &\|\nabla_x\mathcal L(x,\eta)-\nabla_x\mathcal L(x,\eta')\|\\
 &=\left\|\frac1n\sum_{i=1}^n\nabla\ell_i(x)
 \left[\phi'\!\left(\frac{\ell_i(x)-G\eta}{\lambda_0}\right)
 -\phi'\!\left(\frac{\ell_i(x)-G\eta'}{\lambda_0}\right)\right]\right\|\\
 &\leq\frac1n\sum_{i=1}^n\|\nabla\ell_i(x)\|\frac{MG}{\lambda_0}|\eta-\eta'|
 \leq\frac{MG^2}{\lambda_0}|\eta-\eta'|.
\end{align*}
The same calculations for each sample prove Lemma~\ref{L_2 continuous}.
\end{proof}

\begin{proposition}\label{f17:ds-transfer}
For every $x\in\mathbb R^d$ and $\eta\in\mathcal K_\eta$,
\[
 |\nabla_\eta\mathcal L(x,\eta)|\leq B_\eta,
 \qquad\|\nabla\mathcal F(x)\|\leq G,
 \qquad\|\nabla\mathcal F(x)\|^2\leq2\|\nabla\mathcal L(x,\eta)\|^2.
\]
Moreover, $\eta_0=U/G$ gives $\Delta:=\mathcal L(x_0,\eta_0)-\inf\mathcal L\leq U$.
\end{proposition}
\begin{proof}
For fixed $x$, put $a_i(\eta)=\phi'((\ell_i(x)-G\eta)/\lambda_0)$. Then
\begin{align*}
 |\nabla_\eta\mathcal L(x,\eta)|
 &=G\left|1-\frac1n\sum_i a_i(\eta)\right|\leq B_\eta,\\
 \nabla_\eta\mathcal L(x,0)&\leq0\leq\nabla_\eta\mathcal L(x,U/G)
 \quad\Longrightarrow\quad
 \exists\eta^*\in\mathcal K_\eta:\ \frac1n\sum_i a_i(\eta^*)=1.
\end{align*}
With $a_i^*=a_i(\eta^*)$, Fenchel equality and weak duality give
\[
 \mathcal F(x)\leq\min_\eta\mathcal L(x,\eta)
 =\mathcal L(x,\eta^*)
 =\frac1n\sum_i[a_i^*\ell_i(x)-\lambda_0\psi(a_i^*)]
 \leq\mathcal F(x).
\]
Any two minimizing $\eta$ have the same $a_i^*$, since all $a_i$ are nonincreasing and their sums equal $n$. Thus the envelope theorem on $\mathcal K_\eta$ gives
\begin{align*}
 \nabla\mathcal F(x)&=\frac1n\sum_i a_i^*\nabla\ell_i(x),
 &\|\nabla\mathcal F(x)\|&\leq\frac G n\sum_i a_i^*=G,\\
 \|\nabla_x\mathcal L(x,\eta)-\nabla\mathcal F(x)\|
 &\leq\frac G n\sum_i|a_i(\eta)-a_i^*|
 =G\left|\frac1n\sum_i a_i(\eta)-1\right|
 =|\nabla_\eta\mathcal L(x,\eta)|,\\
 \|\nabla\mathcal F(x)\|^2
 &\leq2\|\nabla_x\mathcal L(x,\eta)\|^2+2|\nabla_\eta\mathcal L(x,\eta)|^2.
\end{align*}
Finally, $\phi(s)\geq s$ and $\phi(s)\leq0$ for $s\leq0$, so
\[
 0\leq\inf\mathcal L\leq\mathcal L(x_0,U/G)\leq U,
 \qquad\Delta\leq U.
\]
\end{proof}

Formal statement of Theorem~\ref{main:1}.
\begin{theorem}\label{main1full}\label{f17:ds-rate}
Under the preceding conditions and \eqref{f17:ds-schedule}--\eqref{f17:ds-noise},
DP Double-SPIDER is $(\varepsilon,\delta)$-DP. For an independently uniform
$\tau\in\{1,\ldots,T\}$,
\begin{equation}\label{f17:ds-rate-bound}
 \mathbb E\|\nabla\mathcal F(x_\tau)\|\leq K_D/q\leq K_D A_D^{2/3}.
\end{equation}
\end{theorem}
Let $\eta_0=U/G$, $z_t=(x_t,\eta_t)$, $y_t=(x_t,\eta_{t+1})$, and
\[
 c_D=4096,\quad H_D=\log(16n^2/\delta)\log(8/\delta),\qquad
 \nu_D(K,k)=\frac{c_DH_D}{\varepsilon^2}\max\{k^{-2},K/n^2\}.
\]
For $A_D<1$, set
\begin{equation}\label{f17:ds-schedule}
 \begin{gathered}
 A_D=\frac{\sqrt{(d+1)H_D}}{n\varepsilon},\quad
 q=\lceil A_D^{-2/3}\rceil,\quad T=q^2,\quad R=T/q=q,\\
 N_1=N_3=B=\min\{n,\lceil\max\{q^2,n/\sqrt q\}\rceil\},\qquad
 N_2=N_4=b=\lceil\max\{q,n/q\}\rceil,\\
 \alpha_t=\beta_t=\alpha=\frac1{2L_D\sqrt{2+8c_D}},\qquad
 K_D^2=40U/\alpha+G^2+48B_D^2(1+8c_D).
 \end{gathered}
\end{equation}
The scalar and vector noise variances are, respectively,
\begin{equation}\label{f17:ds-noise}
 \begin{aligned}
 \sigma_1^2&=B_\eta^2\nu_D(R,N_1),&
 \sigma_3^2&=B_x^2\nu_D(R,N_3),\\
 \sigma_{2,t}^2&=L_D^2\|z_t-z_{t-1}\|^2\nu_D(T,N_2),&
 \sigma_{4,t}^2&=L_D^2\|y_t-y_{t-1}\|^2\nu_D(T,N_4).
 \end{aligned}
\end{equation}
\begin{proof}
\noindent\textbf{Privacy Proof:}
With $C_1=B_\eta$, $C_3=B_x$, $C_{2,t}=L_D\|z_t-z_{t-1}\|$,
and $C_{4,t}=L_D\|y_t-y_{t-1}\|$, the pointwise bounds give
\[
 \|u_i\|\leq C
 \ \Longrightarrow\ 
 \operatorname{Clip}\!\left(k^{-1}\sum_{i\in\mathcal B}u_i,C\right)
 =k^{-1}\sum_{i\in\mathcal B}u_i,\qquad
 \operatorname{sens}\!\left(k^{-1}\sum_{i\in\mathcal B}u_i\right)\leq2C/k.
\]
Condition on the complete preceding transcript, including the current $g_t$
before the $x$ query. The query points and thresholds are fixed, and its fresh
batch and Gaussian are independent. Theorem~\ref{f17:privacy}, with four
streams of caps $(R,T,R,T)$ and batch sizes $(N_1,N_2,N_3,N_4)$, applies
to \eqref{f17:ds-noise}. The recursion, projection, and output selection
preserve $(\varepsilon,\delta)$-DP.

\noindent\textbf{Utility Proof:}
Write $a_t=\partial_\eta\mathcal L(z_t)$, $b_t=\nabla_x\mathcal L(y_t)$ and
\[
 e_{\eta,t}=g_t-a_t,\quad e_{x,t}=v_t-b_t,\quad
 \eta_{t+1}=\Pi_{\mathcal K_\eta}(\eta_t-\alpha g_t),\quad
 r_t=\frac{\eta_t-\eta_{t+1}}\alpha,\quad
 W_t=r_t^2+\|v_t\|^2,\quad E_t=e_{\eta,t}^2+\|e_{x,t}\|^2.
\]
Convexity in $\eta$ gives a root $\eta_t^*\in\mathcal K_\eta$ of
$\partial_\eta\mathcal L(x_t,\cdot)$. For $L_\eta=MG^2/\lambda_0$
and $\alpha\leq1/L_\eta$,
\[
 \eta_t-\alpha a_t\in
 [\min\{\eta_t,\eta_t^*\},\max\{\eta_t,\eta_t^*\}]
 \subseteq\mathcal K_\eta,\qquad
 |r_t-a_t|\leq|e_{\eta,t}|,\qquad g_tr_t\geq r_t^2.
\]
Since $\mathcal L$ is $L_D$-smooth on the strip, we have that
\begin{equation}\label{smooth_eta}
 \begin{aligned}
 \mathbb E\mathcal L(x_t,\eta_{t+1})
 &\leq\mathbb E[\mathcal L(x_t,\eta_t)-\alpha\langle a_t,r_t\rangle
                      +(L_D\alpha^2/2)r_t^2]\\
 &=\mathbb E[\mathcal L(x_t,\eta_t)-\alpha\langle g_t,r_t\rangle
                 +\alpha\langle e_{\eta,t},r_t\rangle+(L_D\alpha^2/2)r_t^2]\\
 &\leq\mathbb E\mathcal L(x_t,\eta_t)
       -(3\alpha/4-L_D\alpha^2/2)\mathbb Er_t^2
       +\alpha\mathbb Ee_{\eta,t}^2\\
 &\leq\mathbb E\mathcal L(x_t,\eta_t)-\frac\alpha2\mathbb Er_t^2
                                      +\alpha\mathbb Ee_{\eta,t}^2.
 \end{aligned}
\end{equation}
Similarly, for the update of $x$, we have that
\begin{equation}\label{smooth_x}
 \begin{aligned}
 \mathbb E\mathcal L(x_{t+1},\eta_{t+1})
 &\leq\mathbb E[\mathcal L(x_t,\eta_{t+1})-\alpha\langle b_t,v_t\rangle
                       +(L_D\alpha^2/2)\|v_t\|^2]\\
 &=\mathbb E[\mathcal L(x_t,\eta_{t+1})-\alpha\|v_t\|^2
          +\alpha\langle e_{x,t},v_t\rangle+(L_D\alpha^2/2)\|v_t\|^2]\\
 &\leq\mathbb E\mathcal L(x_t,\eta_{t+1})
       -(3\alpha/4-L_D\alpha^2/2)\mathbb E\|v_t\|^2
       +\alpha\mathbb E\|e_{x,t}\|^2\\
 &\leq\mathbb E\mathcal L(x_t,\eta_{t+1})-\frac\alpha2\mathbb E\|v_t\|^2
                                      +\alpha\mathbb E\|e_{x,t}\|^2.
 \end{aligned}
\end{equation}
Combining \eqref{smooth_eta} and \eqref{smooth_x}, and summing from $0$ to $T-1$, gives
\[
 \frac\alpha2\sum_{t=0}^{T-1}\mathbb EW_t
 \leq\mathcal L(z_0)-\mathbb E\mathcal L(z_T)
       +\alpha\sum_{t=0}^{T-1}\mathbb EE_t
 \leq\Delta+\alpha\sum_{t=0}^{T-1}\mathbb EE_t,\qquad \Delta\leq U.
\]
Let
\begin{equation}\label{f17:ds-variance}
 \begin{gathered}
 \chi_k=\frac{n-k}{k(n-1)},\quad
 V_\eta=B_\eta^2\chi_{N_1}+\sigma_1^2,\quad
 V_x=B_x^2\chi_{N_3}+d\sigma_3^2,\quad V_0=V_\eta+V_x,\\
 \Gamma_\eta=L_D^2[\chi_{N_2}+\nu_D(T,N_2)],\qquad
 \Gamma_x=L_D^2[\chi_{N_4}+d\nu_D(T,N_4)],\qquad
 \theta=q\alpha^2(\Gamma_\eta+\Gamma_x).
 \end{gathered}
\end{equation}
According to Lemma~\ref{bound} below, for $\theta\leq1/4$,
\begin{equation}\label{f17:ds-summed}
 \begin{aligned}
 (1/2-\theta)\sum_{t=0}^{T-1}\mathbb EW_t
 &\leq\Delta/\alpha+TV_0,\\
 \frac1T\sum_{t=0}^{T-1}\mathbb EW_t
 &\leq\frac{4\Delta}{\alpha T}+4V_0,\\
 \frac1T\sum_{t=0}^{T-1}\mathbb EE_t
 &\leq V_0+\frac\theta T\sum_{t=0}^{T-1}\mathbb EW_t
 \leq\frac\Delta{\alpha T}+2V_0.
 \end{aligned}
\end{equation}
Moreover,
\[
 \begin{aligned}
 a_t^2&\leq2r_t^2+2e_{\eta,t}^2,\\
 \|\nabla_x\mathcal L(z_t)\|^2
 &\leq2\|b_t\|^2+2L_D^2\alpha^2r_t^2
 \leq4\|v_t\|^2+4\|e_{x,t}\|^2+2L_D^2\alpha^2r_t^2,\\
 \frac1T\sum_{t=0}^{T-1}\mathbb E\|\nabla\mathcal L(z_t)\|^2
 &\leq\frac4T\sum_{t=0}^{T-1}\mathbb E(W_t+E_t)
 \leq\frac{20\Delta}{\alpha T}+24V_0.
 \end{aligned}
\]
Using $\|\nabla\mathcal F(x)\|^2\leq2\|\nabla\mathcal L(x,\eta)\|^2$
and $\|\nabla\mathcal F(x)\|\leq G$ yields
\begin{equation}\label{f17:ds-master}
 \begin{aligned}
 \mathbb E\|\nabla\mathcal F(x_\tau)\|^2
 &=\frac1T\sum_{t=1}^T\mathbb E\|\nabla\mathcal F(x_t)\|^2\\
 &\leq\frac2T\sum_{t=0}^{T-1}\mathbb E\|\nabla\mathcal L(z_t)\|^2+\frac{G^2}T
 \leq\frac{40\Delta/\alpha+G^2}{T}+48V_0.
 \end{aligned}
\end{equation}
By \eqref{f17:ds-schedule}, $A_D\geq1/n$ and
\[
 \begin{gathered}
 2\leq q\leq\lceil n^{2/3}\rceil\leq n-1,\quad
 b\leq n-1,\quad B\leq n,\quad T=q^2\leq n^2,\quad q^3A_D^2\leq8,\\
 q\chi_b\leq1,\quad\chi_B\leq q^{-2},\quad
 \max\{b^{-2},T/n^2\}=q^2/n^2,\quad
 \max\{B^{-2},R/n^2\}=q/n^2,\\
 q(\Gamma_\eta+\Gamma_x)
 \leq L_D^2(2+c_DA_D^2q^3)\leq L_D^2(2+8c_D),\quad\theta\leq1/4,\\
 V_0\leq B_D^2(q^{-2}+c_DA_D^2q)
       \leq B_D^2(1+8c_D)/q^2.
 \end{gathered}
\]
Here $b\leq n-1$ follows from $q\leq n-1$ and $n/q\leq n/2$;
$\chi_B\leq q^{-2}$ follows from $B=n$ or $B\geq q^2$.
Substituting into \eqref{f17:ds-master} gives
\[
 \begin{aligned}
 \mathbb E\|\nabla\mathcal F(x_\tau)\|
 &\leq\sqrt{\mathbb E\|\nabla\mathcal F(x_\tau)\|^2}
 \leq\frac{\sqrt{40U/\alpha+G^2+48B_D^2(1+8c_D)}}q\\
 &\leq K_D A_D^{2/3}
 =K_D\left(\frac{\sqrt{(d+1)\log(16n^2/\delta)\log(8/\delta)}}{n\varepsilon}\right)^{2/3}.
 \end{aligned}
\]
\end{proof}

In the following, we bound the two gradient-estimation errors.
\begin{lemma}\label{bound}
With \eqref{f17:ds-variance}, for every epoch starting at $t_0$,
\[
 \begin{aligned}
 \sum_{t=t_0}^{t_0+q-1}\mathbb Ee_{\eta,t}^2
 &\leq qV_\eta+q\alpha^2\Gamma_\eta
             \sum_{j=t_0+1}^{t_0+q-1}\mathbb E(\|v_{j-1}\|^2+r_{j-1}^2),\\
 \sum_{t=t_0}^{t_0+q-1}\mathbb E\|e_{x,t}\|^2
 &\leq qV_x+q\alpha^2\Gamma_x
             \sum_{j=t_0+1}^{t_0+q-1}\mathbb E(\|v_{j-1}\|^2+r_j^2).
 \end{aligned}
\]
Consequently,
\begin{equation}\label{f17:ds-error-sum}
 \sum_{t=0}^{T-1}\mathbb EE_t\leq TV_0+
 q\alpha^2(\Gamma_\eta+\Gamma_x)\sum_{t=0}^{T-1}\mathbb EW_t.
\end{equation}
\end{lemma}
\begin{proof}
Write $\mathcal L(z;\mathcal B)=|\mathcal B|^{-1}\sum_{i\in\mathcal B}h_i(z)$.
Let $\mathbb E_{\eta,t}$ condition on the preceding transcript, and let
$\mathbb E_{x,t}$ additionally condition on the current $g_t$.
For a fresh uniform subset $\mathcal B$ of size $k$,
\[
 \mathbb E\!\left[\left\|\frac1k\sum_{i\in\mathcal B}u_i-\bar u\right\|^2
        \middle|\mathscr H\right]
 =\chi_k\frac1n\sum_i\|u_i-\bar u\|^2
 \leq\chi_k\frac1n\sum_i\|u_i\|^2,\qquad \bar u=\frac1n\sum_i u_i.
\]
If $t\bmod q=0$, $g_t$ is conditionally unbiased, and
\[
 \mathbb Ee_{\eta,t}^2
 =\mathbb E\!\left[\chi_{N_1}\frac1n\sum_i
                 |\partial_\eta h_i(z_t)-a_t|^2+\sigma_1^2\right]
 \leq B_\eta^2\chi_{N_1}+\sigma_1^2=V_\eta.
\]
Otherwise, set
$D_{\eta,t}=\partial_\eta\mathcal L(z_t;\mathcal B_{2,t})-
\partial_\eta\mathcal L(z_{t-1};\mathcal B_{2,t})-a_t+a_{t-1}$.
For the scalar Gaussian $\zeta_{2,t}$,
\begin{equation}\label{singlebound_eta}
 \begin{aligned}
 \mathbb Ee_{\eta,t}^2
 &=\mathbb E\!\left[\left(\partial_\eta\mathcal L(z_t;\mathcal B_{2,t})
       -\partial_\eta\mathcal L(z_{t-1};\mathcal B_{2,t})+g_{t-1}-a_t+\zeta_{2,t}\right)^2\right]\\
 &=\mathbb E\mathbb E_{\eta,t}(e_{\eta,t-1}+D_{\eta,t}+\zeta_{2,t})^2\\
 &=\mathbb Ee_{\eta,t-1}^2+\mathbb ED_{\eta,t}^2+\mathbb E\sigma_{2,t}^2,
 \end{aligned}
\end{equation}
since $\mathbb E_{\eta,t}D_{\eta,t}=\mathbb E_{\eta,t}\zeta_{2,t}=0$ and
the batch and Gaussian are conditionally independent. Moreover,
\begin{equation}\label{iteration_bound_eta}
 \begin{aligned}
 \mathbb ED_{\eta,t}^2+\mathbb E\sigma_{2,t}^2
 &\leq\chi_{N_2}\mathbb E\frac1n\sum_i
       |\partial_\eta h_i(z_t)-\partial_\eta h_i(z_{t-1})|^2
       +\mathbb E\sigma_{2,t}^2\\
 &\leq L_D^2[\chi_{N_2}+\nu_D(T,N_2)]\mathbb E\|z_t-z_{t-1}\|^2\\
 &=\Gamma_\eta\alpha^2\mathbb E(\|v_{t-1}\|^2+r_{t-1}^2).
 \end{aligned}
\end{equation}
Combining these equations, for $t_0\bmod q=0$ and $t_0\leq t<t_0+q$,
\[
 \mathbb Ee_{\eta,t}^2\leq V_\eta+
 \Gamma_\eta\alpha^2\sum_{j=t_0+1}^t\mathbb E(\|v_{j-1}\|^2+r_{j-1}^2).
\]
We then focus on the estimate of the gradient to $x$. If $t\bmod q=0$,
\[
 \mathbb E\|e_{x,t}\|^2
 =\mathbb E\!\left[\chi_{N_3}\frac1n\sum_i
                   \|\nabla_xh_i(y_t)-b_t\|^2+d\sigma_3^2\right]
 \leq B_x^2\chi_{N_3}+d\sigma_3^2=V_x.
\]
Otherwise, $y_{t-1}=(x_{t-1},\eta_t)$ and
$D_{x,t}=\nabla_x\mathcal L(y_t;\mathcal B_{4,t})-
\nabla_x\mathcal L(y_{t-1};\mathcal B_{4,t})-b_t+b_{t-1}$.
For the vector Gaussian $\zeta_{4,t}$,
\begin{equation}\label{singlebound_x}
 \begin{aligned}
 \mathbb E\|e_{x,t}\|^2
 &=\mathbb E\|\nabla_x\mathcal L(y_t;\mathcal B_{4,t})
       -\nabla_x\mathcal L(y_{t-1};\mathcal B_{4,t})+v_{t-1}-b_t+\zeta_{4,t}\|^2\\
 &=\mathbb E\mathbb E_{x,t}\|e_{x,t-1}+D_{x,t}+\zeta_{4,t}\|^2\\
 &=\mathbb E\|e_{x,t-1}\|^2+\mathbb E\|D_{x,t}\|^2+d\mathbb E\sigma_{4,t}^2,
 \end{aligned}
\end{equation}
since $\mathbb E_{x,t}D_{x,t}=\mathbb E_{x,t}\zeta_{4,t}=0$ and
the batch and Gaussian are conditionally independent. Consequently,
\begin{equation}\label{iteration_bound_x}
 \begin{aligned}
 \mathbb E\|D_{x,t}\|^2+d\mathbb E\sigma_{4,t}^2
 &\leq\chi_{N_4}\mathbb E\frac1n\sum_i
       \|\nabla_xh_i(y_t)-\nabla_xh_i(y_{t-1})\|^2+d\mathbb E\sigma_{4,t}^2\\
 &\leq L_D^2[\chi_{N_4}+d\nu_D(T,N_4)]\mathbb E\|y_t-y_{t-1}\|^2\\
 &=\Gamma_x\alpha^2\mathbb E(\|v_{t-1}\|^2+r_t^2),\\
 \mathbb E\|e_{x,t}\|^2
 &\leq V_x+\Gamma_x\alpha^2
              \sum_{j=t_0+1}^t\mathbb E(\|v_{j-1}\|^2+r_j^2).
 \end{aligned}
\end{equation}
All second moments are finite by induction from the bounded summands and
the conditional variances in \eqref{f17:ds-noise}. Summing within each epoch,
\[
 \begin{aligned}
 \sum_{t=t_0}^{t_0+q-1}\mathbb EE_t
 &\leq qV_0+\alpha^2\sum_{j=t_0+1}^{t_0+q-1}(t_0+q-j)
    \mathbb E[\Gamma_\eta(\|v_{j-1}\|^2+r_{j-1}^2)
              +\Gamma_x(\|v_{j-1}\|^2+r_j^2)]\\
 &\leq qV_0+q\alpha^2(\Gamma_\eta+\Gamma_x)
                      \sum_{t=t_0}^{t_0+q-1}\mathbb EW_t.
 \end{aligned}
\]
Summing over $t_0\in\{0,q,\ldots,T-q\}$ proves \eqref{f17:ds-error-sum}.
\end{proof}

\subsection{DP Recursive-SPIDER}\label{f17:kl-settings}
Next, here is the full statement of Theorem~\ref{thm:RSDRO}:
\begin{theorem}\label{ma17:kl-full}
For KL divergence, under Appendix~\ref{f17:conditions} and the settings below with $A<1$, Algorithm~\ref{RSDRO} with $\lambda_t\equiv\lambda_0$ is $(\varepsilon,\delta)$-DP and satisfies
\[
 \mathbb E\|\nabla\mathcal F(x_\tau)\|\leq\sqrt{C_*}\,r.
\]
\end{theorem}

\noindent\textbf{Parameter settings.}
Let $w_t=(x_t^\top,\lambda_0)^\top$, $g_i(w)=e^{\ell(x;\xi_i)/\lambda}$, $g(w)=n^{-1}\sum_i g_i(w)$, $J_i=\nabla_wg_i$, $J=n^{-1}\sum_iJ_i$ and $f_\lambda(s)=\lambda\log s$. Define $\Psi(w)=\lambda\log g(w)+\lambda\rho$ for $\lambda>0$. For $u_n=(1/n,\ldots,1/n)$ and $p_i^*=e^{\ell(x;\xi_i)/\lambda_0}/\sum_j e^{\ell(x;\xi_j)/\lambda_0}$,
\begin{align*}
 \sum_i p_i\ell(x;\xi_i)-\lambda_0D_{\rm KL}(p\Vert u_n)
 &=\lambda_0\log g(x,\lambda_0)-\lambda_0D_{\rm KL}(p\Vert p^*),\\
 \mathcal F(x)&=\lambda_0\log g(x,\lambda_0)=\Psi(x,\lambda_0)-\lambda_0\rho,\qquad
 0\leq\mathcal F(x)\leq U.
\end{align*}
Write $\Delta=\mathcal F(x_0)-\inf_x\mathcal F(x)\leq U$ and $m=d+1$. With the loss bounds in Appendix~\ref{f17:conditions}, set
\begin{align}
 E&=e^{U/\lambda_0},\qquad V_g=(E-1)^2/4,\qquad
 P=E\sqrt{G^2/\lambda_0^2+U^2/\lambda_0^4},\nonumber\\
 Q_g&=E\left(\frac{L_\ell}{\lambda_0}+\frac{2G}{\lambda_0^2}+\frac{2U}{\lambda_0^3}
                 +\frac{G^2}{\lambda_0^2}+\frac{U^2}{\lambda_0^4}\right),\label{ma17:kl-constants}\\
 L_F&=L_\ell+G^2/\lambda_0+U^2/\lambda_0^3,\qquad
 C_q=4\lambda_0^2,\qquad C_s=4\lambda_0^2P^2+2.\nonumber
\end{align}
Indeed, on the slice $\lambda=\lambda_0$, with $a_i=(\nabla\ell_i/\lambda_0,-\ell_i/\lambda_0^2)^\top$,
\begin{align*}
 &1\leq g_i\leq E,\quad \operatorname{Var}_i(g_i)\leq V_g,\quad
 J_i=g_i a_i,\quad\|J_i\|\leq P,\\
 &\|D_xa_i\|\leq L_\ell/\lambda_0+G/\lambda_0^2,\quad
 \|a_i\|\|\nabla\ell_i\|/\lambda_0\leq G^2/\lambda_0^2+U^2/\lambda_0^4,\\
 &\|D_xJ_i\|\leq E\left(L_\ell/\lambda_0+G/\lambda_0^2+G^2/\lambda_0^2+U^2/\lambda_0^4\right)\leq Q_g,\\
 &\|J_i(w_{t+1})-J_i(w_t)\|\leq Q_g\|x_{t+1}-x_t\|,\qquad
 \nabla^2\mathcal F=\sum_i p_i^*\nabla^2\ell_i+\lambda_0^{-1}\operatorname{Cov}_{p^*}(\nabla\ell_i),
 \quad \|\nabla^2\mathcal F\|\leq L_F.
\end{align*}
Use $\nu=\nu_6$ from Theorem~\ref{f17:privacy}, namely
\begin{equation}\label{ma17:kl-noise}
 Q_{\rm priv}=\frac{9216}{\varepsilon^2}\log\frac{24n^2}{\delta}\log\frac{12}{\delta},\qquad
 \nu(K,k)=Q_{\rm priv}\max\{k^{-2},K/n^2\}.
\end{equation}
For $n\geq8$, $0<\varepsilon\leq1$ and $0<\delta<1$, choose
\begin{equation}\label{to17:kl-rate-parameter}
 A=\sqrt{mQ_{\rm priv}}/n<1,\qquad r=A^{2/3},\qquad T=\lceil n^2/m\rceil,
\end{equation}
\begin{equation}\label{ma17:kl-schedule}
 \begin{gathered}
 q=\lceil rT\rceil,\quad b=\lceil\sqrt m\rceil,\quad
 B=\min\{n,\lceil\max\{r^{-2},n\sqrt r\}\rceil\},\quad \beta=r^2/Q_{\rm priv},\\
 \overline H=9C_qQ_g^2+4C_sP^2,\quad
 c_\eta=\min\{(2L_F)^{-1},(2\sqrt{\overline H})^{-1}\},\quad \eta=c_\eta r/Q_{\rm priv},\\
 C_*=10U/c_\eta+36\{C_qP^2+C_s(V_g+E^2)\}.
 \end{gathered}
\end{equation}
Let $R=\lceil T/q\rceil$, $\nu_R=\nu(R,B)$, $\nu_D=\nu(T,b)$, $\nu_S=\nu(T,1)$ and $\nu_0=\nu(1,B)$. Set
\begin{align*}
 b_1=b_3&=B,\quad b_2=b_4=b,\quad C_1=C_3=P,\quad
 C_{2,t}=C_{4,t}=Q_g\|x_t-x_{t-1}\|,\quad C_5=E,\\
 \sigma_1^2=\sigma_3^2&=\nu_RP^2,\qquad
 \sigma_{2,t}^2=\sigma_{4,t}^2=\nu_DQ_g^2\|x_t-x_{t-1}\|^2,\\
 \sigma_{s,\mathrm{init}}^2&=\nu_0E^2,\qquad
 \sigma_{s,t}^2=\nu_S(P\|x_{t+1}-x_t\|+\beta E)^2.
\end{align*}
All batches are independent uniform subsets without replacement; indices are not released. Noises are conditionally independent centered Gaussians with the displayed variances. Initialize
\[
 s_0=\Pi_{[1,E]}\left(B^{-1}\sum_{i\in\mathcal B_0}g_i(w_0)+Z_0\right),
 \qquad Z_0\sim\mathcal N(0,\sigma_{s,\mathrm{init}}^2),
\]
using an independent batch $\mathcal B_0$. At $t=0,\ldots,T-1$, refresh $v_t,u_t$ if $t\bmod q=0$; otherwise add the noisy difference to $v_{t-1},u_{t-1}$, respectively. Then set
\begin{equation}\label{ma17:kl-update}
 q_t=(v_t^\top,u_t)^\top,\quad z_t=f_{\lambda_0}'(s_t)q_t+(0_d^\top,\log s_t+\rho)^\top,
 \quad y_t=(z_t)_x=\lambda_0v_t/s_t,\quad x_{t+1}=x_t-\eta y_t.
\end{equation}
Finally draw one fresh uniform index $i_t$ and use the scalar update in Lemma~\ref{13}. The clipping bounds above hold sample by sample, so clipping is inactive; set $\operatorname{Clip}(0,0)=0$. Return $x_\tau$ with independent uniform $\tau\in\{0,\ldots,T-1\}$.

\begin{proof}[Privacy]
For each Gaussian query, conditional on the preceding transcript, its replacement sensitivity is at most $2C/k$. The six query sequences have $(K,k,C)$ equal to
\[
 \underbrace{(R,B,P),\ (R,B,P)}_{v,u\text{ refreshes}},\qquad
 \underbrace{(T,b,Q_g\|x_t-x_{t-1}\|),\ (T,b,Q_g\|x_t-x_{t-1}\|)}_{v,u\text{ differences}},
\]
\[
 (1,B,E),\qquad (T,1,P\|x_{t+1}-x_t\|+\beta E),
\]
where the last bound follows from
$|g_i(w_{t+1})-(1-\beta)g_i(w_t)|\leq P\|x_{t+1}-x_t\|+\beta E$.
Theorem~\ref{f17:privacy}, with variance $C^2\nu(K,k)$, gives $(\varepsilon,\delta)$-DP. Previous trackers, projections and parameter updates are postprocessing. The bounds $K\leq n^2$ and $1\leq k\leq n$ are verified below.
\end{proof}

\noindent\textbf{Utility Proof:}
First we start with the first technical lemma. All second moments below are finite by induction: $s_t\in[1,E]$, $\|y_t\|\leq\lambda_0\|v_t\|$, and conditional noise variances are bounded by constants times $1+\|x_{t+1}-x_t\|^2$.
\begin{lemma}\label{lem:z-Psi}
Let $\mathscr X_t=s_t-g(w_t)$. For every $t=0,\ldots,T-1$,
\[
 \|y_t-\nabla\mathcal F(x_t)\|^2\leq\|z_t-\nabla\Psi(w_t)\|^2
 \leq C_s|\mathscr X_t|^2+C_q\bigl(\|v_t-\nabla_xg(w_t)\|^2+|u_t-\partial_\lambda g(w_t)|^2\bigr).
\]
\end{lemma}
\begin{proof}
By simple expansion, writing $g_t=g(w_t)$, it holds that
\begin{align}
 \|z_t-\nabla\Psi(w_t)\|^2
 &=\|f_{\lambda_0}'(g_t)\nabla_xg(w_t)-f_{\lambda_0}'(s_t)v_t\|^2\nonumber\\
 &\quad+|f_{\lambda_0}'(g_t)\partial_\lambda g(w_t)-f_{\lambda_0}'(s_t)u_t+\log g_t-\log s_t|^2\nonumber\\
 &\leq2\|f_{\lambda_0}'(g_t)J(w_t)-f_{\lambda_0}'(s_t)q_t\|^2+2|g_t-s_t|^2,
 \label{eq:kl-error-split}
\end{align}
where $|\log a-\log b|\leq|a-b|$ for $a,b\geq1$. Applying $f_{\lambda_0}'(s)=\lambda_0/s$ and $\|J(w_t)\|\leq P$, we obtain
\begin{align}
 &\|f_{\lambda_0}'(g_t)J(w_t)-f_{\lambda_0}'(s_t)q_t\|^2\nonumber\\
 &=\|[f_{\lambda_0}'(g_t)-f_{\lambda_0}'(s_t)]J(w_t)+f_{\lambda_0}'(s_t)[J(w_t)-q_t]\|^2\nonumber\\
 &\leq2\|[f_{\lambda_0}'(g_t)-f_{\lambda_0}'(s_t)]J(w_t)\|^2
       +2|f_{\lambda_0}'(s_t)|^2\|J(w_t)-q_t\|^2\nonumber\\
 &\leq2\lambda_0^2P^2|\mathscr X_t|^2+2\lambda_0^2\|q_t-J(w_t)\|^2.
 \label{eq:kl-error-bound}
\end{align}
Noting $\|q_t-J(w_t)\|^2=\|v_t-\nabla_xg(w_t)\|^2+|u_t-\partial_\lambda g(w_t)|^2$ and combining \eqref{eq:kl-error-split} and \eqref{eq:kl-error-bound}, we have
\[
 \|z_t-\nabla\Psi(w_t)\|^2\leq(4\lambda_0^2P^2+2)|\mathscr X_t|^2+4\lambda_0^2\|q_t-J(w_t)\|^2.
\]
Finally, $(\nabla\Psi(w_t))_x=\nabla\mathcal F(x_t)$. This completes the proof.
\end{proof}

\begin{lemma}\label{13}
Let $d_t=x_{t+1}-x_t$, $S_t=\mathbb E|\mathscr X_t|^2$, $a_s=2P^2(1+\nu_S)$ and $v_s=2(V_g+E^2\nu_S)$. Then
\begin{equation}\label{ma17:kl-scalar-recursion}
 S_{t+1}\leq(1-\beta)^2S_t+a_s\mathbb E\|d_t\|^2+\beta^2v_s.
\end{equation}
\end{lemma}
\begin{proof}
Since
\[
 \widetilde s_{t+1}=g_{i_t}(w_{t+1})+(1-\beta)(s_t-g_{i_t}(w_t))+\zeta_t,
 \qquad s_{t+1}=\Pi_{[1,E]}(\widetilde s_{t+1}),\qquad
 \zeta_t\sim\mathcal N(0,\sigma_{s,t}^2),
\]
condition on the history $\mathscr H_t$ immediately before drawing $i_t,\zeta_t$; thus $w_t,w_{t+1},s_t$ are fixed. Write $\mathbb E_t=\mathbb E[\,\cdot\mid\mathscr H_t]$ and
$\delta_i=g_i(w_{t+1})-g_i(w_t)$, $\bar\delta=g(w_{t+1})-g(w_t)$. It holds that
\begin{align*}
 \mathbb E_t|s_{t+1}-g(w_{t+1})|^2
 &\leq\mathbb E_t|\widetilde s_{t+1}-g(w_{t+1})|^2\\
 &=\mathbb E_t\big|(1-\beta)\mathscr X_t+\beta[g_{i_t}(w_{t+1})-g(w_{t+1})]
                  +(1-\beta)(\delta_{i_t}-\bar\delta)+\zeta_t\big|^2\\
 &=(1-\beta)^2|\mathscr X_t|^2+
 \mathbb E_t\big|\beta[g_{i_t}(w_{t+1})-g(w_{t+1})]+(1-\beta)(\delta_{i_t}-\bar\delta)\big|^2+\sigma_{s,t}^2,
\end{align*}
where $\mathbb E_t[g_{i_t}(w_{t+1})-g(w_{t+1})]=\mathbb E_t[\delta_{i_t}-\bar\delta]=\mathbb E_t\zeta_t=0$. Noting $\mathbb E_t[\delta_{i_t}\bar\delta]=|\bar\delta|^2$, we have
\begin{align*}
 \mathbb E_t|\delta_{i_t}-\bar\delta|^2
 &=\mathbb E_t|\delta_{i_t}|^2+|\bar\delta|^2-2\mathbb E_t[\delta_{i_t}\bar\delta]\\
 &=\mathbb E_t|\delta_{i_t}|^2-|\bar\delta|^2
 \leq\mathbb E_t|g_{i_t}(w_{t+1})-g_{i_t}(w_t)|^2\leq P^2\|d_t\|^2.
\end{align*}
Combining the above inequalities and taking expectations, we obtain
\begin{align*}
 S_{t+1}&\leq(1-\beta)^2S_t+2\beta^2V_g+2(1-\beta)^2P^2\mathbb E\|d_t\|^2+\mathbb E\sigma_{s,t}^2\\
 &\leq(1-\beta)^2S_t+2\beta^2V_g+2P^2\mathbb E\|d_t\|^2
                         +2\nu_SP^2\mathbb E\|d_t\|^2+2\nu_S\beta^2E^2\\
 &=(1-\beta)^2S_t+a_s\mathbb E\|d_t\|^2+\beta^2v_s.
\end{align*}
\end{proof}

\begin{lemma}\label{ma17:kl-coupled-errors}
Let $\chi_a=(n-a)/(a(n-1))$ and
\[
 V_R=P^2(\chi_B+m\nu_R),\quad c_q=Q_g^2(\chi_b+m\nu_D),\quad
 S_{\rm init}=V_g\chi_B+E^2\nu_0,
\]
\[
 M_T=\sum_{t=0}^{T-1}\mathbb E\|d_t\|^2,\qquad
 \mathcal E_T=\sum_{t=0}^{T-1}\mathbb E\|y_t-\nabla\mathcal F(x_t)\|^2.
\]
Then
\begin{equation}\label{ma17:kl-coupled-sum}
 \mathcal E_T\leq TB_{\rm err}+HM_T,\qquad
 H=C_qqc_q+C_sa_s/\beta,\qquad
 B_{\rm err}=C_qV_R+\frac{C_sS_{\rm init}}{\beta T}+C_s\beta v_s.
\end{equation}
\end{lemma}
\begin{proof}
For a uniform size-$a$ subset $\mathcal B$ and $\bar h=n^{-1}\sum_i h_i$,
\[
 \mathbb E\left\|a^{-1}\sum_{i\in\mathcal B}h_i-\bar h\right\|^2
 =\chi_a\frac1n\sum_i\|h_i-\bar h\|^2.
\]
Write $Q_t=\mathbb E\|q_t-J(w_t)\|^2$. Conditional application to the two independently sampled blocks gives
\[
 Q_t\leq
 \begin{cases}
 V_R,&t\bmod q=0,\\
 Q_{t-1}+c_q\mathbb E\|d_{t-1}\|^2,&t\bmod q\ne0.
 \end{cases}
\]
The centered difference innovations have zero conditional cross terms. Thus, for $t_0=q\lfloor t/q\rfloor$,
\[
 Q_t\leq V_R+c_q\sum_{j=t_0+1}^{t}\mathbb E\|d_{j-1}\|^2,
 \qquad \sum_{t=0}^{T-1}Q_t\leq TV_R+qc_qM_T.
\]
Also $S_0\leq S_{\rm init}$. Taking summation in Lemma~\ref{13}, we have
\begin{align*}
 [1-(1-\beta)^2]\sum_{t=0}^{T-1}S_t
 &\leq S_0-S_T+a_sM_T+T\beta^2v_s,\\
 \sum_{t=0}^{T-1}S_t&\leq S_{\rm init}/\beta+(a_s/\beta)M_T+T\beta v_s.
\end{align*}
Combining the above two inequalities with Lemma~\ref{lem:z-Psi},
\[
 \mathcal E_T\leq C_s\sum_{t=0}^{T-1}S_t+C_q\sum_{t=0}^{T-1}Q_t
 \leq C_sS_{\rm init}/\beta+T(C_s\beta v_s+C_qV_R)+(C_sa_s/\beta+C_qqc_q)M_T.
\]
\end{proof}

\begin{proof}[Completion of the utility proof]
Since $\mathcal F$ is smooth with parameter $L_F$,
\[
 \mathcal F(x_{t+1})\leq\mathcal F(x_t)+\langle\nabla\mathcal F(x_t),d_t\rangle+\frac{L_F}{2}\|d_t\|^2.
\]
Combining this with $d_t=-\eta y_t$ and writing $e_t=y_t-\nabla\mathcal F(x_t)$, we get
\[
 \langle e_t,d_t\rangle+\left(\frac1\eta-\frac{L_F}{2}\right)\|d_t\|^2
 \leq\mathcal F(x_t)-\mathcal F(x_{t+1}).
\]
That is,
\begin{align*}
 \left(\frac1\eta-\frac{L_F}{2}\right)\|d_t\|^2
 &\leq\mathcal F(x_t)-\mathcal F(x_{t+1})-\langle e_t,d_t\rangle\\
 &\leq\mathcal F(x_t)-\mathcal F(x_{t+1})+\eta\|e_t\|^2+\frac1{4\eta}\|d_t\|^2,
\end{align*}
where the last inequality uses Young's inequality. Rearranging and summing across $t=0,\ldots,T-1$, we have
\begin{align}
 \left(\frac3{4\eta}-\frac{L_F}{2}\right)\sum_{t=0}^{T-1}\|d_t\|^2
 &\leq\mathcal F(x_0)-\mathcal F(x_T)+\eta\sum_{t=0}^{T-1}\|e_t\|^2\nonumber\\
 &\leq\Delta+\eta\sum_{t=0}^{T-1}\|e_t\|^2.\label{eq:kl-descent-sum}
\end{align}
For $\eta L_F\leq1/2$ and $\eta^2H\leq1/4$, this and Lemma~\ref{ma17:kl-coupled-errors} yield
\[
 M_T\leq2\eta\Delta+2\eta^2\mathcal E_T,\qquad
 \mathcal E_T\leq TB_{\rm err}+2\eta H\Delta+2\eta^2H\mathcal E_T,
 \qquad \frac{\mathcal E_T}{T}\leq2B_{\rm err}+\frac{\Delta}{\eta T}.
\]
Therefore, using $\nabla\mathcal F(x_t)=y_t-e_t$,
\begin{equation}\label{ma17:kl-master}
 \mathbb E\|\nabla\mathcal F(x_\tau)\|^2
 \leq\frac{2M_T}{\eta^2T}+\frac{2\mathcal E_T}{T}
 \leq\frac{4\Delta}{\eta T}+\frac{6\mathcal E_T}{T}
 \leq\frac{10\Delta}{\eta T}+12B_{\rm err}.
\end{equation}
It remains to substitute the parameter settings. Since $r^3=mQ_{\rm priv}/n^2<1$ and $Q_{\rm priv}\geq1$,
\begin{gather*}
 1\leq q\leq T\leq n^2,\quad R\leq2/r,\quad q\leq3Q_{\rm priv}/r^2,
 \quad1\leq b,B\leq n,\quad0<\beta\leq1,\\
 \chi_B\leq r^2\quad(B=n\text{ or }B\geq r^{-2}),\qquad
 B\geq n\sqrt r,\quad b\geq\sqrt m,\quad T\leq n^2/m+1,\\
 m\nu_R\leq2r^2,\qquad \nu_0\leq r^2/m,\qquad
 m\nu_D\leq2Q_{\rm priv},\qquad \nu_S=Q_{\rm priv}.
\end{gather*}
Consequently,
\begin{align*}
 V_R&\leq3P^2r^2,\quad S_{\rm init}\leq(V_g+E^2)r^2,\quad c_q\leq3Q_g^2Q_{\rm priv},\\
 a_s&\leq4P^2Q_{\rm priv},\quad v_s\leq2(V_g+E^2)Q_{\rm priv},\quad H\leq Q_{\rm priv}^2\overline H/r^2,\\
 \eta L_F&\leq1/2,\quad\eta^2H\leq1/4,\quad\beta T\geq1/r,\quad\eta T\geq c_\eta/r^2,\\
 B_{\rm err}&\leq3\{C_qP^2+C_s(V_g+E^2)\}r^2,\qquad\frac{\Delta}{\eta T}\leq\frac{U}{c_\eta}r^2.
\end{align*}
Combining the above inequalities, we obtain
\begin{align*}
 \mathbb E\|\nabla\mathcal F(x_\tau)\|^2&\leq C_*r^2,\\
 \mathbb E\|\nabla\mathcal F(x_\tau)\|
 &\leq\sqrt{C_*}\,r
 =9216^{1/3}\sqrt{C_*}
 \left(\frac{\sqrt{(d+1)\log(24n^2/\delta)\log(12/\delta)}}{n\varepsilon}\right)^{2/3}.
\end{align*}
This completes the proof.
\end{proof}

\section{Experimental Results}\label{exp_detail}

\subsection{Computer Resources}
We use AMD Ryzen Threadripper 3970X and 2 NVIDIA GeForce RTX 3090 (24 GB VRAM each, 48 GB total) with RAM: 125 GiB total.

\noindent\textbf{Parameter Setup}
We will use ResNet20 for our model. To make a fair comparison, we will consider DRO with KL-divergence. The constrained parameter $\rho$ is set to be $ 0.5$ and  $\lambda_0$ is set to $1e-3$. 
For all experiments, the batch size is 128. 
 In addition, we set privacy budget $\varepsilon=\{0.5,1,5,10\}$ with $\delta=\frac{1}{n^{1.1}}$.  
\subsection{Test AUC Analysis}\label{auc}
\begin{figure*}[htbp]
    \centering
    \includegraphics[width=0.45\linewidth]{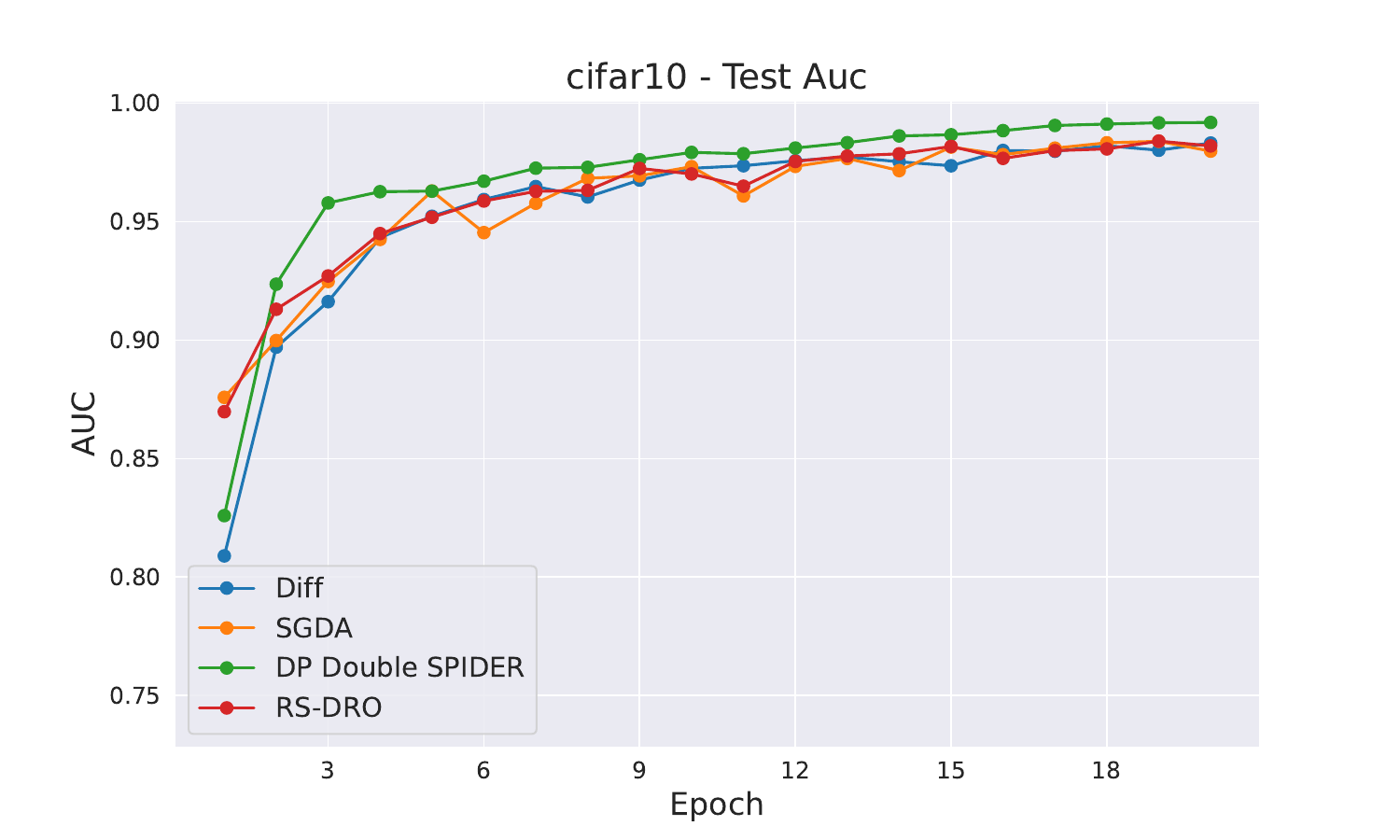}
    \hfill
    \includegraphics[width=0.45\linewidth]{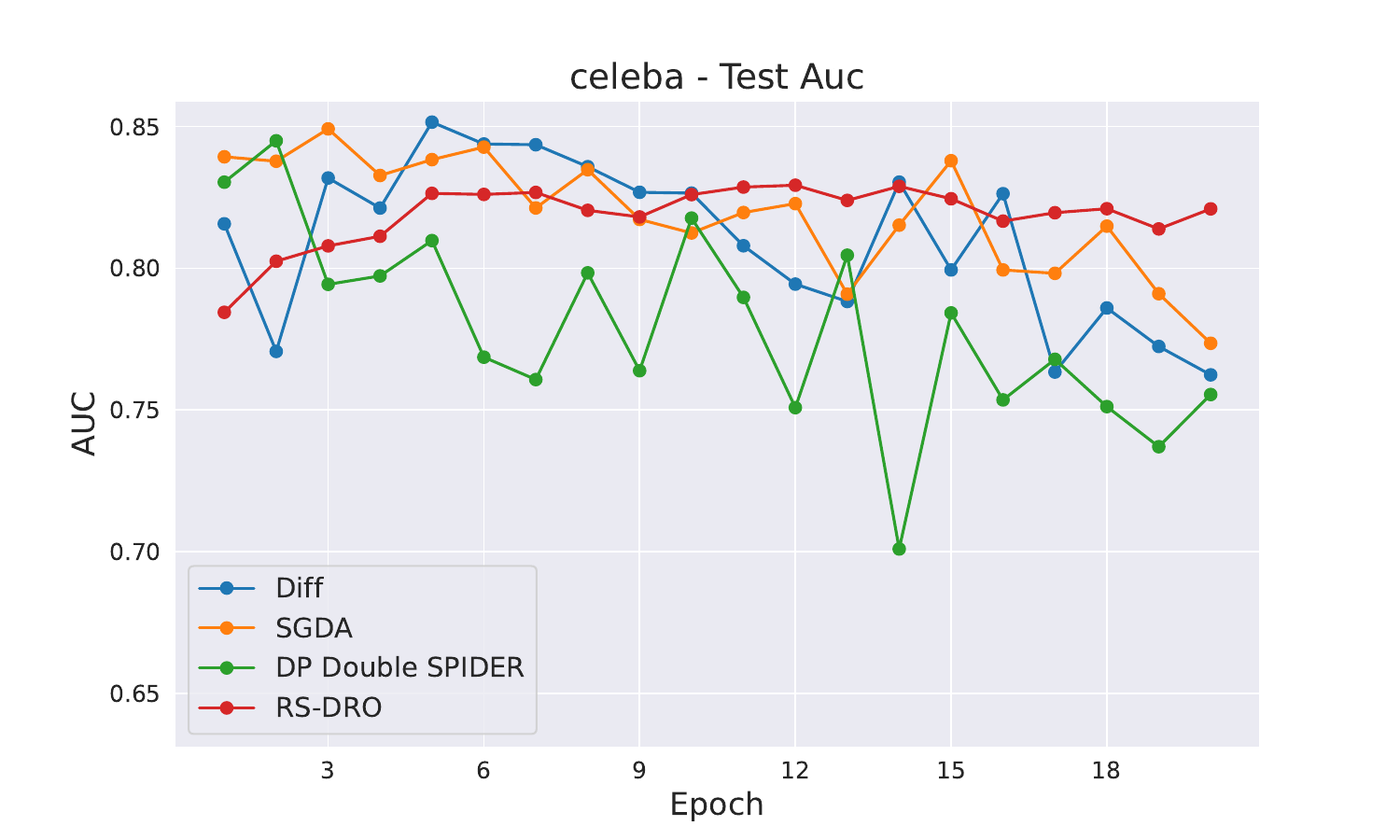}
    
    \vspace{0.5cm}
    
    \includegraphics[width=0.45\linewidth]{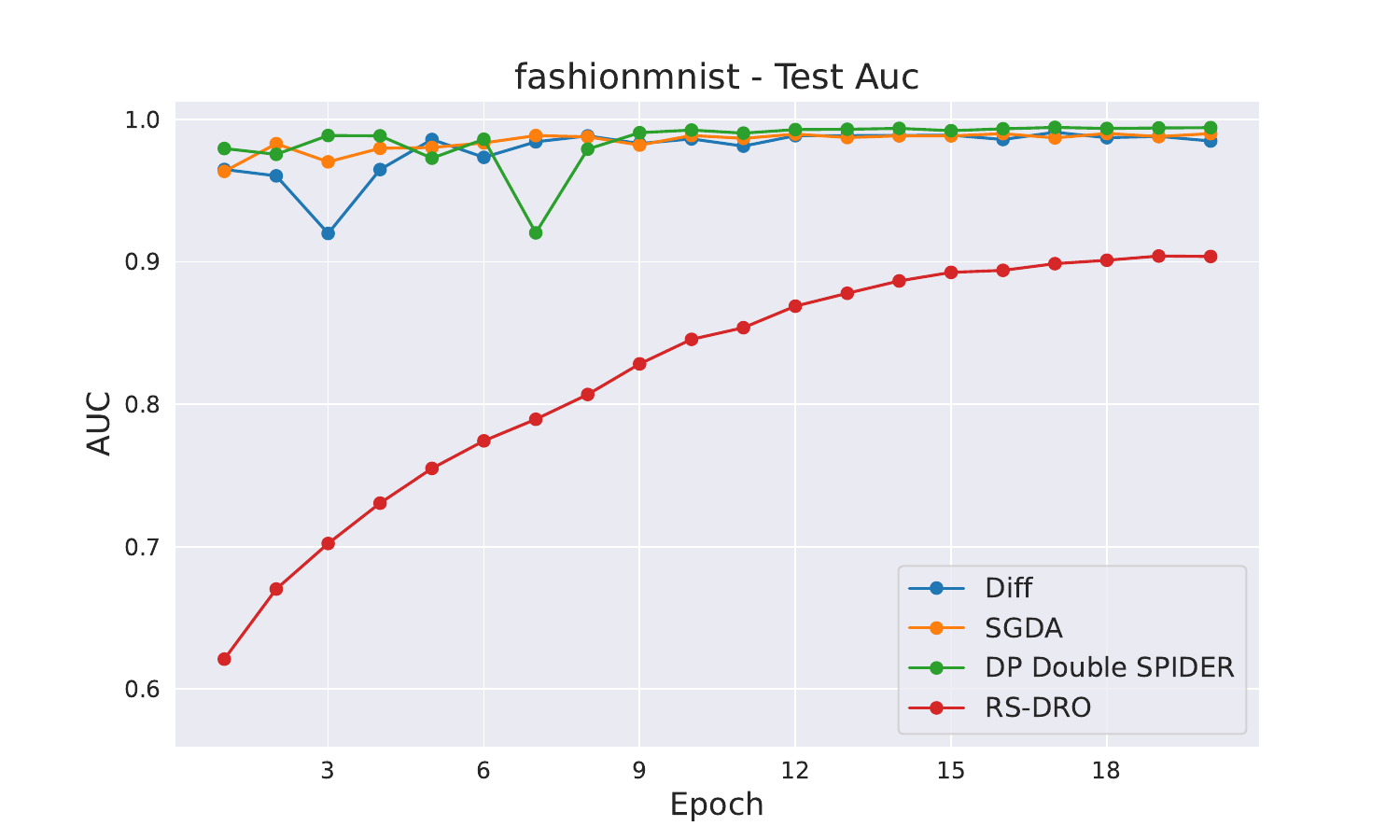}
    \hfill
    \includegraphics[width=0.45\linewidth]{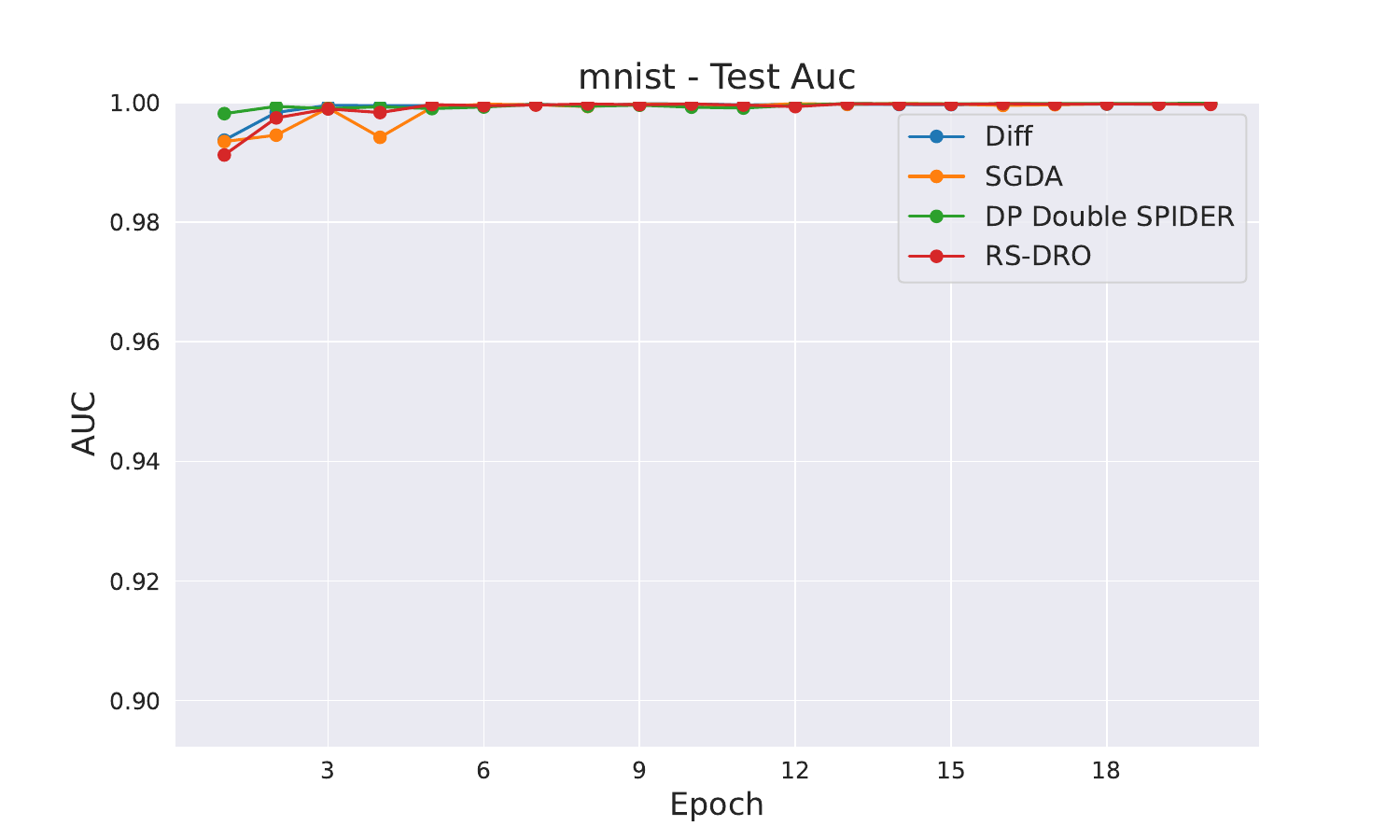}
    
    \caption{Test AUC Results: The performances of four algorithms on CIFAR10-ST, CelebA, Fashion-MNIST, MNIST-ST respectively}
    \label{fig:test_auc}
\end{figure*}

Figure~\ref{fig:test_auc} presents the Area Under the Curve (AUC) performance of four algorithms across different datasets. The AUC metric is particularly valuable as it measures the model's ability to distinguish between classes regardless of the classification threshold, providing a comprehensive view of classifier performance.

Across the four datasets (CIFAR10-ST, CelebA, Fashion-MNIST, and MNIST-ST), we observe distinct performance patterns. The MNIST-ST dataset generally exhibits the highest AUC scores, which is expected given its relatively simple binary classification task with well-separated digit classes. Fashion-MNIST presents a more challenging scenario due to the visual similarity between certain clothing categories, resulting in comparatively lower AUC values. The CelebA dataset, dealing with facial attribute classification, demonstrates intermediate performance, while CIFAR10-ST shows varying convergence patterns depending on the algorithm employed.

The comparative analysis reveals that different algorithms exhibit varying degrees of robustness across datasets. Some algorithms demonstrate consistent performance across all datasets, suggesting good generalization capabilities, while others show dataset-specific strengths. The convergence behavior also varies, with some algorithms achieving stable AUC values earlier in training, while others require more iterations to reach optimal performance.

\newpage
\subsection{Test F1 Score Analysis}\label{f1}

\begin{figure*}[htbp]
    \centering
    \includegraphics[width=0.45\linewidth]{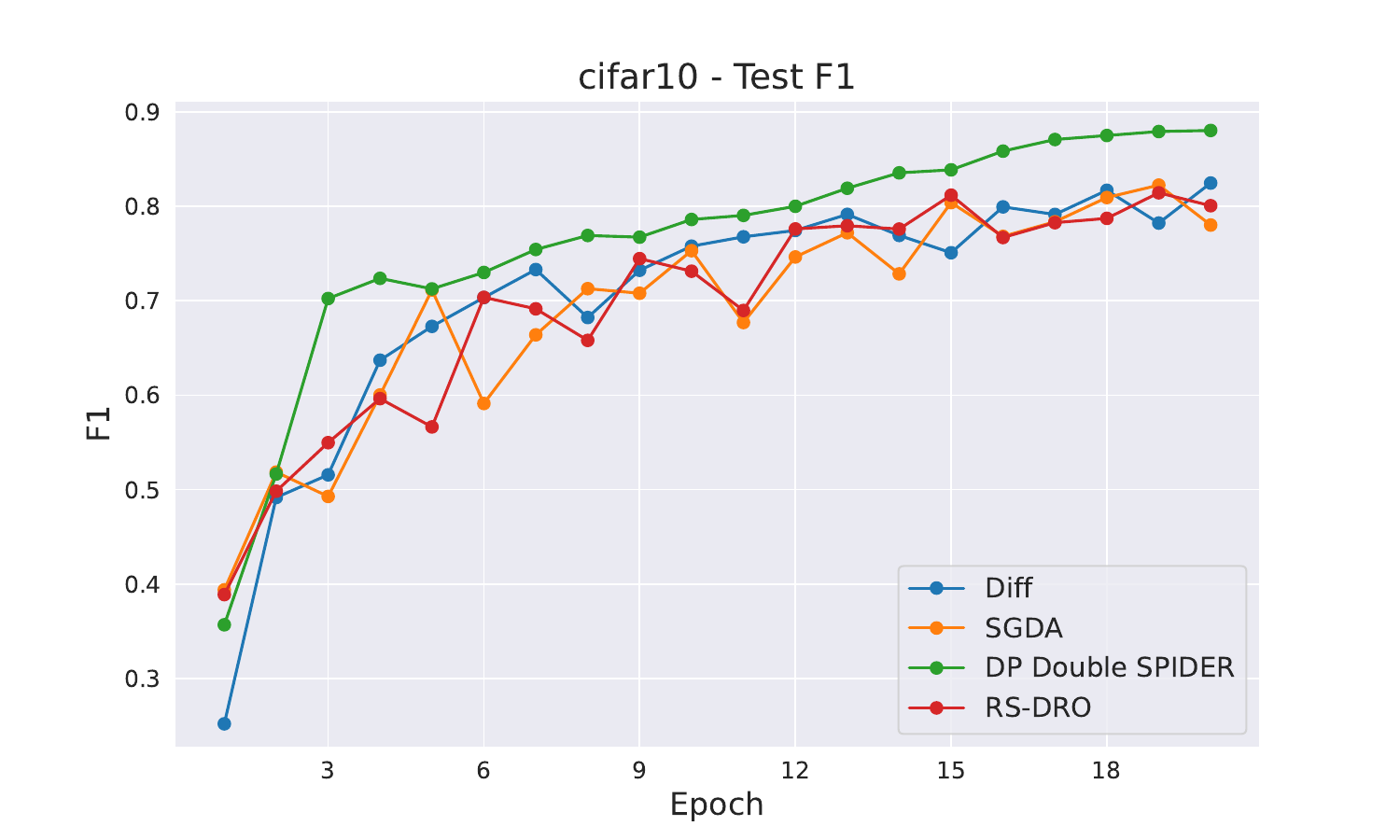}
    \hfill
    \includegraphics[width=0.45\linewidth]{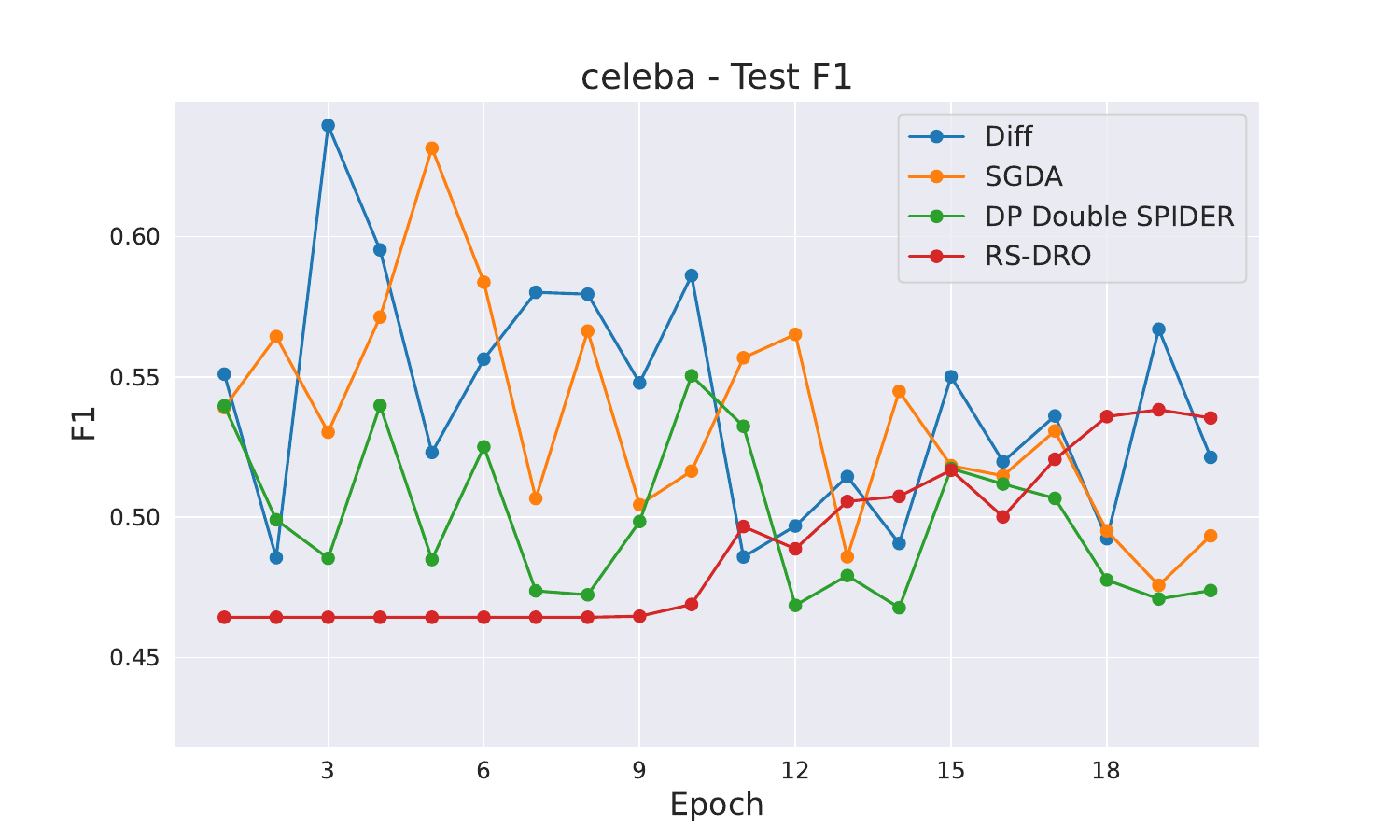}
    
    \vspace{0.5cm}
    
    \includegraphics[width=0.45\linewidth]{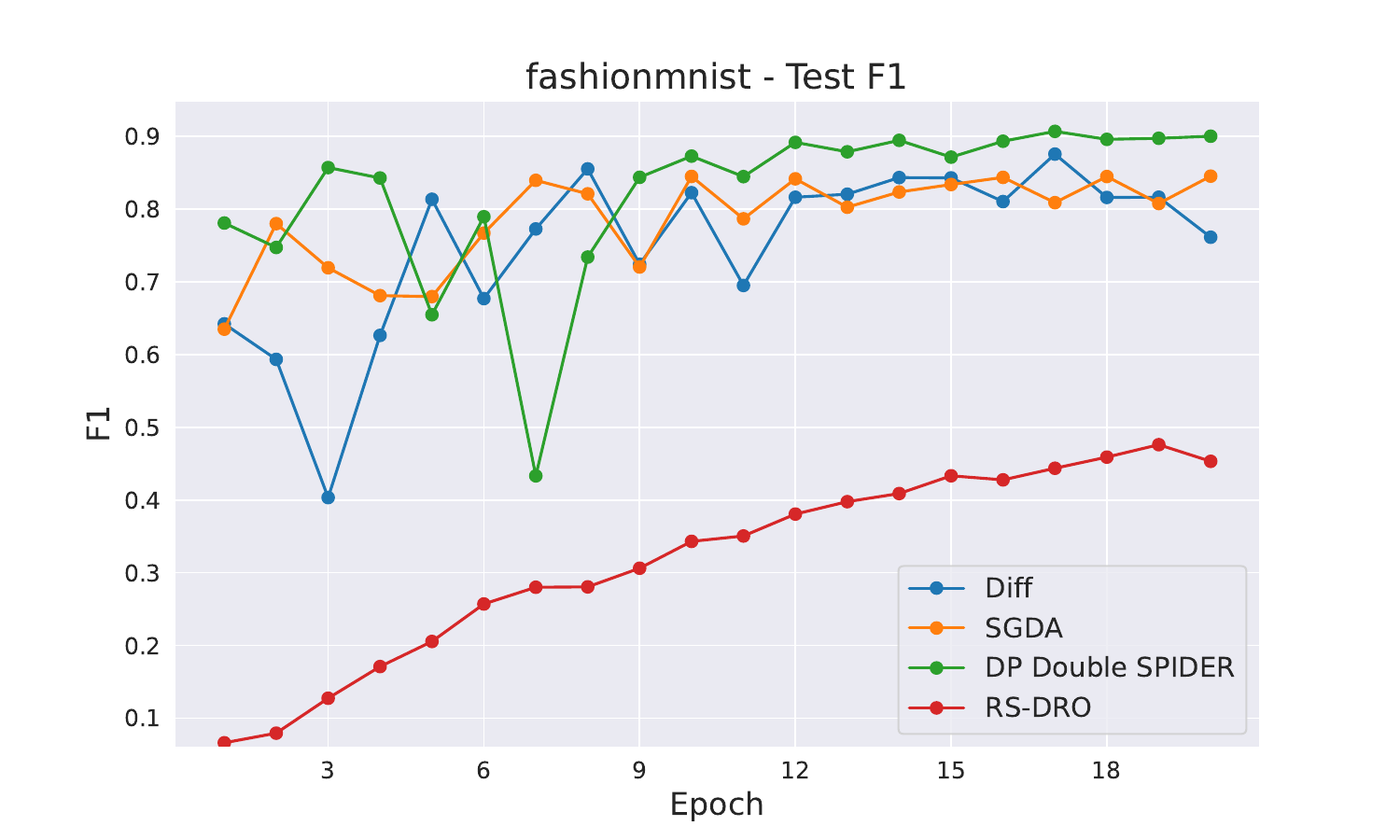}
    \hfill
    \includegraphics[width=0.45\linewidth]{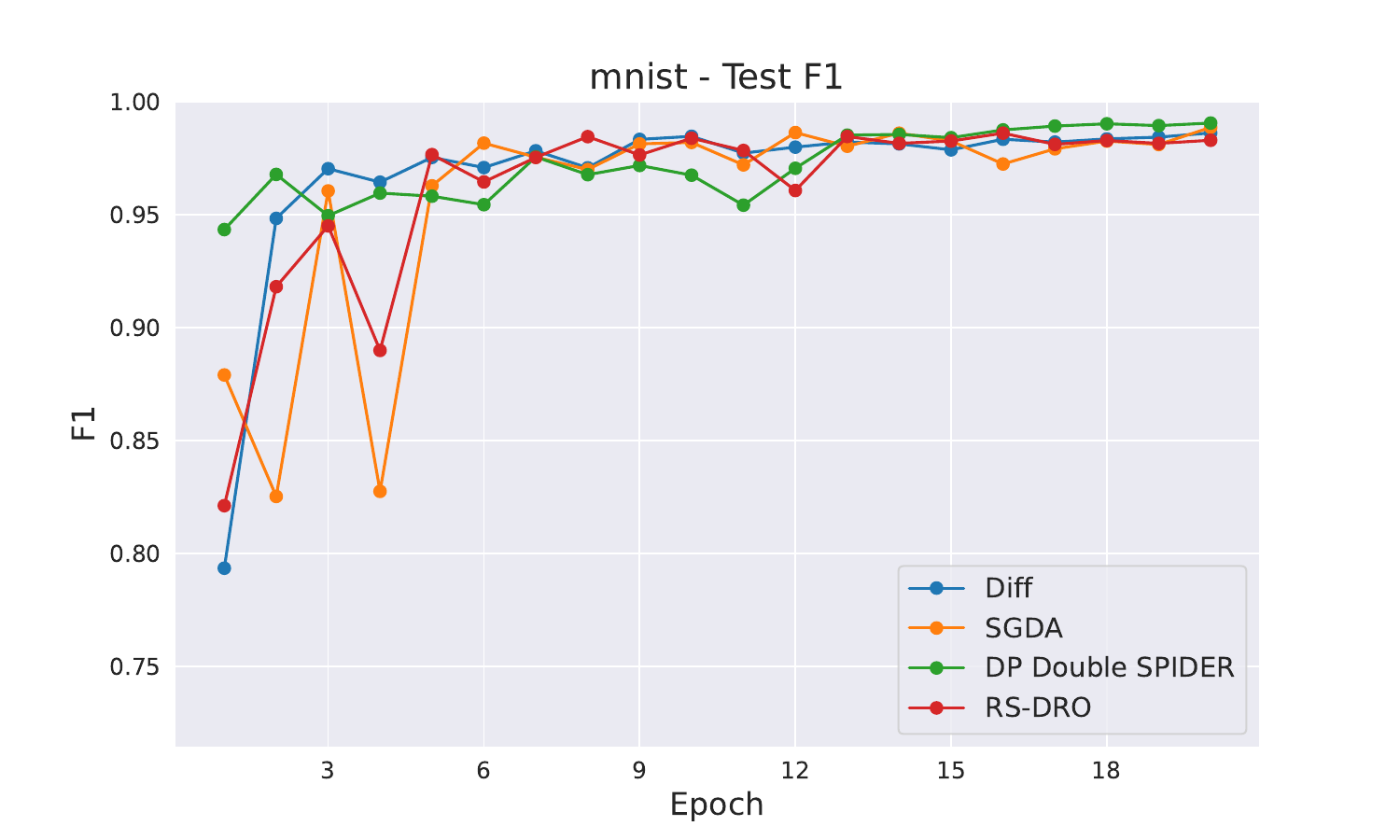}
    
    \caption{Test F1 Score Results: The performances of four algorithms on CIFAR10-ST, CelebA, Fashion-MNIST, MNIST-ST respectively}
    \label{fig:test_f1}
\end{figure*}


Figure~\ref{fig:test_f1} presents the F1 score results, which provide a harmonic mean of precision and recall. This metric is particularly important when dealing with imbalanced datasets or when both false positives and false negatives carry significant costs.

The F1 scores reveal nuanced performance differences that may not be immediately apparent from accuracy alone. For  MNIST-ST, the F1 scores closely track accuracy values, confirming robust performance across both metrics. However, for potentially imbalanced scenarios in CelebA (where certain facial attributes may be rare), the F1 score provides critical insights into the model's ability to correctly identify minority class instances without sacrificing precision.

On Fashion-MNIST F1 scores indicate how well algorithms balance precision and recall when distinguishing between similar clothing categories. Lower F1 scores compared to accuracy might suggest that while overall accuracy is reasonable, the model struggles with specific class pairs, leading to either high false positive or false negative rates for certain categories.

CIFAR10-ST F1 scores demonstrate the algorithms' effectiveness in handling the complexity of natural images. The relationship between F1 scores and accuracy across different algorithms reveals which approaches better balance the trade-off between precision and recall. Algorithms showing F1 scores close to their accuracy values indicate balanced performance, while larger gaps suggest potential issues with either over-prediction or under-prediction of certain classes.

Overall, the F1 score analysis complements the AUC and accuracy metrics by highlighting the algorithms' ability to maintain balanced performance across different evaluation criteria, which is essential for practical deployment scenarios.


\subsection{Hyperparameter Sensitivity}\label{parameter}

We conduct additional hyperparameter sensitivity experiments on CIFAR10-ST. CIFAR10-ST is used because it is the most challenging benchmark in our experiments and because the performance differences between methods are most visible on this dataset. Unless otherwise stated, each configuration is averaged over three random seeds, and we report mean $\pm$ standard deviation of test accuracy.

The goal of this appendix is to examine whether the additional algorithm-specific parameters introduced by DP Double-SPIDER and DP Recursive-SPIDER create a practical tuning burden. We focus on the parameters that are specific to the proposed variance-reduced private algorithms and that could reasonably raise sensitivity concerns: clipping thresholds, the SPIDER phase size, and the learning-rate multiplier. We do not perform exhaustive ablations over all implementation constants. Some quantities, such as the privacy budget $(\varepsilon,\delta)$, define the privacy requirement rather than a tuning choice. Other quantities, such as the model architecture, batch size, DRO radius $\rho$, and lower-bound parameter $\lambda_0$, are fixed for fair comparison across algorithms or follow common settings in related DP/DRO experiments. The Gaussian noise scales are also not independent hyperparameters, since they are determined by the privacy accounting once the clipping thresholds, sampling scheme, and privacy parameters are fixed.

All experiments use KL-DRO with radius $\rho=0.5$, lower-bound parameter $\lambda_0=10^{-3}$, batch size $128$, and privacy failure probability $\delta=n^{-1.1}$. Non-swept algorithmic parameters are fixed to their default values.

The sensitivity study requires many repeated runs, so the experiments below use controlled protocols. For DP Double-SPIDER and the clipping comparison with DP-SGDA and PrivateDiff, we use a 20-epoch ResNet20 pretraining protocol on CIFAR10-ST, whose test accuracy is approximately $46\%$. For DP Recursive-SPIDER, we use the training framework of the main paper: full CIFAR10-ST with $n=25{,}500$, ResNet20 with width factor $0.5$, a 12-epoch class-balanced pretrain with test accuracy around $55.5\%$, a $\Psi_0$ warm start, and theory-derived values of $T,b_1,b_2,q$ with the same implementation caps used in the main experiments. Therefore, absolute accuracies should not be compared across the DP Double-SPIDER and DP Recursive-SPIDER sensitivity blocks. The relevant comparison is the variation across hyperparameter values within each controlled protocol.

The main paper reports the uniform clipping sweep for DP Double-SPIDER, where all four clipping thresholds are tied as $C_1=C_2=C_3=C_4=C$. That result directly addresses the main concern that multiple clipping thresholds may require careful tuning. To avoid overlap, this appendix does not repeat that table. Instead, it reports the remaining sensitivity studies: pairwise clipping, random independent clipping, comparison with baselines, and DP Recursive-SPIDER sensitivity to clipping scale, phase size, and learning-rate multiplier.

\textbf{DP Double-SPIDER Clipping Sensitivity} DP Double-SPIDER contains four clipping thresholds. The thresholds $C_1$ and $C_2$ are used for the refresh and recursive update of the dual-variable gradient estimator, while $C_3$ and $C_4$ are used for the corresponding primal-variable gradient estimator. The main paper shows that these four thresholds can be tied to a single value without meaningful performance loss. Here, we further test whether the dual-side and primal-side clipping thresholds need to be tuned separately, whether heterogeneous random clipping values affect performance, and whether DP Double-SPIDER is more sensitive to clipping than the baselines.

\paragraph{Pairwise clipping independence.}
We first test whether the dual-side and primal-side clipping thresholds need to be tuned separately. In the first group, we vary $C_1=C_2$ while fixing $C_3=C_4=1$. In the second group, we vary $C_3=C_4$ while fixing $C_1=C_2=1$. The results are reported in Table~\ref{tab:appendix_pairwise_clip}.

\begin{table*}[t]
\centering
\caption{Pairwise clipping sensitivity of DP Double-SPIDER on CIFAR10-ST. Top: vary the dual-side clipping thresholds $C_1=C_2$ while fixing $C_3=C_4=1$. Bottom: vary the primal-side clipping thresholds $C_3=C_4$ while fixing $C_1=C_2=1$. We report test accuracy (\%) as mean $\pm$ standard deviation over three random seeds.}
\label{tab:appendix_pairwise_clip}
\small
\resizebox{\textwidth}{!}{
\begin{tabular}{ccccccc}
\toprule
\multicolumn{2}{c}{Setting} & \multicolumn{5}{c}{Clip value} \\
\cmidrule(lr){1-2} \cmidrule(lr){3-7}
Group & $\varepsilon$ & 0.25 & 0.5 & 1.0 & 2.0 & 4.0 \\
\midrule
\multirow{3}{*}{Vary $C_1=C_2$}
& $1.0$  & $45.86 \pm 0.04$ & $45.86 \pm 0.05$ & $45.87 \pm 0.05$ & $45.87 \pm 0.04$ & $45.88 \pm 0.04$ \\
& $5.0$  & $46.03 \pm 0.24$ & $46.05 \pm 0.19$ & $46.05 \pm 0.24$ & $46.04 \pm 0.25$ & $46.07 \pm 0.20$ \\
& $10.0$ & $46.07 \pm 0.13$ & $46.04 \pm 0.13$ & $46.05 \pm 0.12$ & $46.05 \pm 0.16$ & $46.06 \pm 0.14$ \\
\midrule
\multirow{3}{*}{Vary $C_3=C_4$}
& $1.0$  & $45.87 \pm 0.06$ & $45.87 \pm 0.05$ & $45.87 \pm 0.05$ & $45.87 \pm 0.05$ & $45.87 \pm 0.04$ \\
& $5.0$  & $46.05 \pm 0.20$ & $46.05 \pm 0.24$ & $46.05 \pm 0.24$ & $46.06 \pm 0.23$ & $46.06 \pm 0.22$ \\
& $10.0$ & $46.03 \pm 0.02$ & $46.10 \pm 0.07$ & $46.05 \pm 0.12$ & $46.09 \pm 0.14$ & $46.09 \pm 0.15$ \\
\bottomrule
\end{tabular}
}
\end{table*}

The pairwise sweep shows that the dual-side and primal-side clipping thresholds do not require separate tuning. Across all rows in Table~\ref{tab:appendix_pairwise_clip}, the largest range of mean accuracy is only $0.07$ percentage points. Varying $C_1=C_2$ while fixing $C_3=C_4=1$ changes accuracy by at most $0.04$ percentage points for $\varepsilon=5$, and varying $C_3=C_4$ while fixing $C_1=C_2=1$ changes accuracy by at most $0.07$ percentage points for $\varepsilon=10$. These variations are smaller than the seed-level standard deviations. Thus, the two sides of the estimator do not behave like independent fragile tuning knobs. This supports the practical implementation rule of using a common clipping value for all four thresholds.

\paragraph{Random independent clipping stress test.}
The pairwise sweep still varies the clipping thresholds in a structured way. To test a more heterogeneous setting, we sample each clipping threshold independently from $[0.5,2.0]$ and evaluate ten random configurations at $\varepsilon=5$. This setting allows the four clipping values to differ from each other rather than being tied.

\begin{table*}[t]
\centering
\caption{Random independent clipping stress test for DP Double-SPIDER on CIFAR10-ST with $\varepsilon=5$. Each configuration samples $C_1,C_2,C_3,C_4$ independently from $[0.5,2.0]$. We report test accuracy (\%) as mean $\pm$ standard deviation over three random seeds.}
\label{tab:appendix_random_clip}
\small
\begin{tabular}{cccccc}
\toprule
Config & $C_1$ & $C_2$ & $C_3$ & $C_4$ & Accuracy (\%) \\
\midrule
1  & 1.46 & 0.54 & 0.91 & 0.83 & $46.05 \pm 0.22$ \\
2  & 1.60 & 1.52 & 1.84 & 0.63 & $46.22 \pm 0.25$ \\
3  & 1.13 & 0.54 & 0.83 & 1.26 & $46.03 \pm 0.24$ \\
4  & 0.54 & 0.80 & 1.47 & 1.32 & $46.06 \pm 0.21$ \\
5  & 0.83 & 1.38 & 1.71 & 0.51 & $46.16 \pm 0.25$ \\
6  & 1.71 & 1.55 & 1.01 & 0.73 & $46.11 \pm 0.23$ \\
7  & 1.94 & 1.00 & 0.64 & 0.65 & $46.08 \pm 0.22$ \\
8  & 1.77 & 1.41 & 1.71 & 1.59 & $46.08 \pm 0.24$ \\
9  & 1.30 & 1.96 & 1.07 & 1.33 & $46.03 \pm 0.21$ \\
10 & 1.74 & 1.43 & 1.79 & 1.37 & $46.09 \pm 0.24$ \\
\midrule
\multicolumn{5}{r}{Across configurations} & mean $=46.09$, std $=0.06$, range $=0.19$ \\
\bottomrule
\end{tabular}
\end{table*}

The random stress test further confirms clipping robustness. Across ten heterogeneous clipping configurations, the mean accuracy is $46.09\%$, the cross-configuration standard deviation of the reported means is only $0.06\%$, and the range is $0.19$ percentage points. Therefore, even when the four clipping thresholds are independently sampled rather than tied, the final accuracy remains stable. This result provides a stronger test than the tied clipping sweep: it shows that the method is not only robust to the scale of a common clipping value, but also robust to moderate heterogeneity among the four clipping thresholds.

\paragraph{Comparison with baselines.}
We next compare clipping sensitivity with DP-SGDA and PrivateDiff. For each method, all clipping thresholds are tied to a common value $C\in\{0.5,1,2,4\}$. This comparison directly tests whether having more clipping thresholds in the proposed method makes it harder to tune than the baselines.

\begin{table*}[t]
\centering
\caption{Clipping sensitivity comparison on CIFAR10-ST. For each method, all clipping thresholds are tied to a common value $C$. We report test accuracy (\%) as mean $\pm$ standard deviation over three random seeds.}
\label{tab:appendix_clip_vs_baselines}
\small
\resizebox{\textwidth}{!}{
\begin{tabular}{cccccc}
\toprule
$\varepsilon$ & Method & $C=0.5$ & $C=1.0$ & $C=2.0$ & $C=4.0$ \\
\midrule
$1.0$ & DP Double-SPIDER & $45.87 \pm 0.04$ & $45.87 \pm 0.05$ & $45.87 \pm 0.04$ & $45.87 \pm 0.04$ \\
$1.0$ & DP-SGDA          & $29.09 \pm 2.20$ & $28.39 \pm 1.81$ & $29.51 \pm 1.28$ & $25.60 \pm 1.71$ \\
$1.0$ & PrivateDiff      & $27.49 \pm 0.91$ & $25.80 \pm 1.51$ & $25.37 \pm 1.16$ & $21.83 \pm 0.98$ \\
\midrule
$5.0$ & DP Double-SPIDER & $46.05 \pm 0.17$ & $46.05 \pm 0.24$ & $46.07 \pm 0.24$ & $46.08 \pm 0.22$ \\
$5.0$ & DP-SGDA          & $30.51 \pm 2.59$ & $29.44 \pm 2.46$ & $30.75 \pm 2.03$ & $28.52 \pm 1.15$ \\
$5.0$ & PrivateDiff      & $28.24 \pm 1.19$ & $27.41 \pm 2.41$ & $26.54 \pm 0.87$ & $27.83 \pm 1.91$ \\
\midrule
$10.0$ & DP Double-SPIDER & $46.02 \pm 0.06$ & $46.05 \pm 0.12$ & $46.10 \pm 0.15$ & $46.11 \pm 0.16$ \\
$10.0$ & DP-SGDA          & $31.33 \pm 0.76$ & $29.24 \pm 1.85$ & $29.35 \pm 3.76$ & $29.32 \pm 3.38$ \\
$10.0$ & PrivateDiff      & $29.09 \pm 2.11$ & $28.00 \pm 0.64$ & $28.44 \pm 1.67$ & $28.52 \pm 1.71$ \\
\bottomrule
\end{tabular}
}
\end{table*}

Table~\ref{tab:appendix_clip_vs_baselines} shows that DP Double-SPIDER is less sensitive to clipping than the baselines, despite having more clipping variables in its theoretical statement. For DP Double-SPIDER, the range across clipping values is $0.00$, $0.03$, and $0.09$ percentage points for $\varepsilon=1,5,10$, respectively. In contrast, the corresponding ranges are $3.91$, $2.23$, and $2.09$ percentage points for DP-SGDA, and $5.66$, $1.70$, and $1.09$ percentage points for PrivateDiff. Therefore, the number of clipping constants in the algorithmic description does not translate into practical tuning difficulty. Empirically, DP Double-SPIDER is more robust to the clipping scale than both baseline methods.

\textbf{DP Recursive-SPIDER Sensitivity} We next evaluate the sensitivity of DP Recursive-SPIDER. Algorithm~3 contains five clipping thresholds: $C_1$--$C_4$ are used for the recursive gradient estimators, and $C_5$ is used for the value tracker $s_t$. In addition, the method has a phase size $q$, which controls the frequency of full-batch refreshes, and a learning-rate parameter. All experiments in this subsection are conducted on CIFAR10-ST with $\varepsilon=5$.

\paragraph{Clipping-scale sensitivity.}
In the implementation of DP Recursive-SPIDER, the clipping thresholds are controlled through the gradient-bound and smoothness constants. We therefore tie them by setting $G=L=C$ and sweep $C\in\{0.5,1,2,4\}$. The results are reported in Table~\ref{tab:appendix_rs_clip}.

\begin{table}[t]
\centering
\caption{DP Recursive-SPIDER under uniform clipping scale $G=L=C$ on CIFAR10-ST with $\varepsilon=5$. We report test accuracy (\%) as mean $\pm$ standard deviation over three random seeds.}
\label{tab:appendix_rs_clip}
\small
\begin{tabular}{ccccc}
\toprule
$C$ & 0.5 & 1.0 & 2.0 & 4.0 \\
\midrule
Accuracy (\%) & $55.97 \pm 1.65$ & $54.81 \pm 2.63$ & $53.62 \pm 3.43$ & $53.53 \pm 3.29$ \\
\bottomrule
\end{tabular}
\end{table}

The mean accuracy ranges from $53.53\%$ to $55.97\%$, giving a range of $2.44$ percentage points and a cross-$C$ standard deviation of $1.153\%$ over the reported means. The results show moderate sensitivity to the clipping scale, especially as $C$ increases. However, the method remains stable over an eight-fold sweep from $C=0.5$ to $C=4.0$: no tested value causes training collapse, and the variation is comparable to the seed-level standard deviations. Thus, DP Recursive-SPIDER benefits from a reasonable default clipping scale, but it does not require precise fine tuning of the clipping threshold.

\paragraph{Phase-size sensitivity.}
The phase size $q$ controls how often DP Recursive-SPIDER performs a refresh step rather than a recursive variance-reduced update. The theory-derived default in this setting is $q_0=5$. We test $q\in\{2,5,10,20\}$, corresponding to approximately half, one, two, and four times the default value.

\begin{table}[t]
\centering
\caption{DP Recursive-SPIDER under varying phase size on CIFAR10-ST with $\varepsilon=5$. The default value is $q_0=5$. We report test accuracy (\%) as mean $\pm$ standard deviation over three random seeds.}
\label{tab:appendix_rs_q}
\small
\begin{tabular}{ccc}
\toprule
Phase-size setting & Actual $q$ & Accuracy (\%) \\
\midrule
$0.5q_0$ & 2  & $55.20 \pm 2.68$ \\
$q_0$    & 5  & $55.38 \pm 2.15$ \\
$2q_0$   & 10 & $55.03 \pm 2.30$ \\
$4q_0$   & 20 & $54.73 \pm 2.57$ \\
\bottomrule
\end{tabular}
\end{table}

The phase-size sweep is highly stable. The range of mean accuracy is $0.65$ percentage points, and the cross-$q$ standard deviation of the reported means is $0.275\%$. The default value $q_0=5$ achieves the highest mean accuracy, but the differences between $q=2,5,10,20$ are small relative to the seed-level standard deviations. This indicates that the refresh period is not a delicate hyperparameter. In practice, the theory-derived value can be used without extensive tuning.

\paragraph{Learning-rate multiplier sensitivity.}
Finally, we vary the learning-rate multiplier. Let $\eta_0$ denote the default step size. We test the effective step size $s\eta_0$ with $s\in\{0.5,1,2\}$.

\begin{table}[t]
\centering
\caption{DP Recursive-SPIDER under learning-rate multiplier $s$ on CIFAR10-ST with $\varepsilon=5$, where the effective step size is $s\eta_0$. We report test accuracy (\%) as mean $\pm$ standard deviation over three random seeds.}
\label{tab:appendix_rs_lr}
\small
\begin{tabular}{cccc}
\toprule
$s$ & 0.5 & 1.0 & 2.0 \\
\midrule
Accuracy (\%) & $56.24 \pm 1.41$ & $55.38 \pm 2.15$ & $53.52 \pm 3.53$ \\
\bottomrule
\end{tabular}
\end{table}

The learning rate has a larger effect than the phase size. Halving the default step size gives the highest mean accuracy, while doubling the default step size lowers the mean accuracy. The range of the reported means is $2.72$ percentage points, and the cross-$s$ standard deviation is $1.392\%$. Nevertheless, all tested settings produce valid training behavior, and the default setting $s=1$ remains close to the best tested configuration within the observed seed-level variability. Thus, the learning rate should be set to a reasonable default, but the method does not require delicate fine tuning.

The sensitivity results support the claim that the additional parameters introduced by the proposed algorithms do not create a practical fine-tuning burden. For DP Double-SPIDER, the evidence is particularly strong. The main paper reports that tying all four clipping thresholds to one common value leads to almost unchanged accuracy over a two-order-of-magnitude sweep. The additional experiments in this appendix further show that the dual-side and primal-side clipping thresholds do not need to be tuned separately, that independently sampled heterogeneous clipping thresholds produce stable accuracy, and that DP Double-SPIDER is less sensitive to the clipping scale than DP-SGDA and PrivateDiff.

These results directly address the concern that DP Double-SPIDER has ``too many parameters.'' Although the theoretical algorithm contains four clipping thresholds, the experiments show that these thresholds can be treated as one practical parameter. Moreover, compared with the baselines, the proposed method is more robust to the clipping value. Therefore, the number of clipping constants in the theoretical description does not imply a larger tuning burden in practice.

For DP Recursive-SPIDER, the phase size $q$ is highly stable around the theory-derived default. The clipping scale and learning-rate multiplier have larger effects than $q$, but the method remains stable over multi-fold sweeps. These results support a simple practical implementation strategy: use tied clipping thresholds, use the theory-derived phase size, and fix standard training choices consistently across methods.

We also emphasize that not every parameter appearing in the algorithms requires a separate ablation. Some quantities are part of the privacy requirement, such as $(\varepsilon,\delta)$. Some are fixed for fair comparison, such as the model architecture, batch size, DRO radius $\rho$, and lower-bound parameter $\lambda_0$. Others, especially the Gaussian noise scales, are determined by the privacy analysis once the clipping thresholds, sampling scheme, and privacy parameters are specified. The sensitivity study therefore focuses on the algorithm-specific choices that could reasonably raise concerns about tuning. Under this targeted evaluation, the proposed methods show stable performance and do not require exhaustive hyperparameter search.

%% file: main.bib
@article{zhou2024differentially,
  title={Differentially private worst-group risk minimization},
  author={Zhou, Xinyu and Bassily, Raef},
  journal={arXiv preprint arXiv:2402.19437},
  year={2024}
}

@inproceedings{dwork2006calibrating,
  title={Calibrating noise to sensitivity in private data analysis},
  author={Dwork, Cynthia and McSherry, Frank and Nissim, Kobbi and Smith, Adam},
  booktitle={Theory of cryptography conference},
  pages={265--284},
  year={2006},
  organization={Springer}
}

@article{zhang2022bring,
  title={Bring your own algorithm for optimal differentially private stochastic minimax optimization},
  author={Zhang, Liang and Thekumparampil, Kiran K and Oh, Sewoong and He, Niao},
  journal={Advances in Neural Information Processing Systems},
  volume={35},
  pages={35174--35187},
  year={2022}
}

@article{cutkosky2019momentum,
  title={Momentum-based variance reduction in non-convex sgd},
  author={Cutkosky, Ashok and Orabona, Francesco},
  journal={Advances in neural information processing systems},
  volume={32},
  year={2019}
}

@inproceedings{abadi2016deep,
  title={Deep learning with differential privacy},
  author={Abadi, Martin and Chu, Andy and Goodfellow, Ian and McMahan, H Brendan and Mironov, Ilya and Talwar, Kunal and Zhang, Li},
  booktitle={Proceedings of the 2016 ACM SIGSAC conference on computer and communications security},
  pages={308--318},
  year={2016}
}

@inproceedings{zhao2023differentially,
  title={Differentially private temporal difference learning with stochastic nonconvex-strongly-concave optimization},
  author={Zhao, Canzhe and Ze, Yanjie and Dong, Jing and Wang, Baoxiang and Li, Shuai},
  booktitle={Proceedings of the Sixteenth ACM International Conference on Web Search and Data Mining},
  pages={985--993},
  year={2023}
}

@inproceedings{bassily2014private,
  title={Private empirical risk minimization: Efficient algorithms and tight error bounds},
  author={Bassily, Raef and Smith, Adam and Thakurta, Abhradeep},
  booktitle={2014 IEEE 55th annual symposium on foundations of computer science},
  pages={464--473},
  year={2014},
  organization={IEEE}
}

@article{fang2018spider,
  title={Spider: Near-optimal non-convex optimization via stochastic path-integrated differential estimator},
  author={Fang, Cong and Li, Chris Junchi and Lin, Zhouchen and Zhang, Tong},
  journal={Advances in neural information processing systems},
  volume={31},
  year={2018}
}

@article{selvi2025differential,
  title={Differential privacy via distributionally robust optimization},
  author={Selvi, Aras and Liu, Huikang and Wiesemann, Wolfram},
  journal={Operations Research},
  year={2025},
  publisher={INFORMS}
}

@article{bassily2019private,
  title={Private stochastic convex optimization with optimal rates},
  author={Bassily, Raef and Feldman, Vitaly and Talwar, Kunal and Guha Thakurta, Abhradeep},
  journal={Advances in neural information processing systems},
  volume={32},
  year={2019}
}

@article{bassily2021differentially,
  title={Differentially private stochastic optimization: New results in convex and non-convex settings},
  author={Bassily, Raef and Guzm{\'a}n, Crist{\'o}bal and Menart, Michael},
  journal={Advances in Neural Information Processing Systems},
  volume={34},
  pages={9317--9329},
  year={2021}
}

@article{choquette2024optimal,
  title={Optimal rates for dp-sco with a single epoch and large batches},
  author={Choquette-Choo, Christopher A and Ganesh, Arun and Thakurta, Abhradeep},
  journal={arXiv e-prints},
  pages={arXiv--2406},
  year={2024}
}

@inproceedings{feldman2020private,
  title={Private stochastic convex optimization: optimal rates in linear time},
  author={Feldman, Vitaly and Koren, Tomer and Talwar, Kunal},
  booktitle={Proceedings of the 52nd Annual ACM SIGACT Symposium on Theory of Computing},
  pages={439--449},
  year={2020}
}

@inproceedings{su2023differentially,
  title={Differentially private stochastic convex optimization in (non)-Euclidean space revisited},
  author={Su, Jinyan and Zhao, Changhong and Wang, Di},
  booktitle={Uncertainty in Artificial Intelligence},
  pages={2026--2035},
  year={2023},
  organization={PMLR}
}

@inproceedings{tao2022private,
  title={Private Stochastic Convex Optimization and Sparse Learning with Heavy-tailed Data Revisited.},
  author={Tao, Youming and Wu, Yulian and Cheng, Xiuzhen and Di Wang},
  booktitle={IJCAI},
  pages={3947--3953},
  year={2022}
}

@article{wang2020empirical,
  title={Empirical risk minimization in the non-interactive local model of differential privacy},
  author={Wang, Di and Gaboardi, Marco and Smith, Adam and Xu, Jinhui},
  journal={Journal of machine learning research},
  volume={21},
  number={200},
  pages={1--39},
  year={2020}
}

@inproceedings{wang2020differentially,
  title={On differentially private stochastic convex optimization with heavy-tailed data},
  author={Wang, Di and Xiao, Hanshen and Devadas, Srinivas and Xu, Jinhui},
  booktitle={International Conference on Machine Learning},
  pages={10081--10091},
  year={2020},
  organization={PMLR}
}

@article{wang2017differentially,
  title={Differentially private empirical risk minimization revisited: Faster and more general},
  author={Wang, Di and Ye, Minwei and Xu, Jinhui},
  journal={Advances in Neural Information Processing Systems},
  volume={30},
  year={2017}
}

@inproceedings{wang2019differentially,
  title={Differentially private empirical risk minimization with non-convex loss functions},
  author={Wang, Di and Chen, Changyou and Xu, Jinhui},
  booktitle={International Conference on Machine Learning},
  pages={6526--6535},
  year={2019},
  organization={PMLR}
}

@article{zhou2020private,
  title={Private stochastic non-convex optimization: Adaptive algorithms and tighter generalization bounds},
  author={Zhou, Yingxue and Chen, Xiangyi and Hong, Mingyi and Wu, Zhiwei Steven and Banerjee, Arindam},
  journal={arXiv preprint arXiv:2006.13501},
  year={2020}
}

@inproceedings{xiao2023theory,
  title={A theory to instruct differentially-private learning via clipping bias reduction},
  author={Xiao, Hanshen and Xiang, Zihang and Wang, Di and Devadas, Srinivas},
  booktitle={2023 IEEE Symposium on Security and Privacy (SP)},
  pages={2170--2189},
  year={2023},
  organization={IEEE}
}

@inproceedings{arora2023faster,
  title={Faster rates of convergence to stationary points in differentially private optimization},
  author={Arora, Raman and Bassily, Raef and Gonz{\'a}lez, Tom{\'a}s and Guzm{\'a}n, Crist{\'o}bal A and Menart, Michael and Ullah, Enayat},
  booktitle={International Conference on Machine Learning},
  pages={1060--1092},
  year={2023},
  organization={PMLR}
}

@article{rafique2022weakly,
  title={Weakly-convex--concave min--max optimization: provable algorithms and applications in machine learning},
  author={Rafique, Hassan and Liu, Mingrui and Lin, Qihang and Yang, Tianbao},
  journal={Optimization Methods and Software},
  volume={37},
  number={3},
  pages={1087--1121},
  year={2022},
  publisher={Taylor \& Francis}
}

@article{duchi2021learning,
  title={Learning models with uniform performance via distributionally robust optimization},
  author={Duchi, John C and Namkoong, Hongseok},
  journal={The Annals of Statistics},
  volume={49},
  number={3},
  pages={1378--1406},
  year={2021},
  publisher={Institute of Mathematical Statistics}
}

@article{levy2020large,
  title={Large-scale methods for distributionally robust optimization},
  author={Levy, Daniel and Carmon, Yair and Duchi, John C and Sidford, Aaron},
  journal={Advances in neural information processing systems},
  volume={33},
  pages={8847--8860},
  year={2020}
}

@article{wang2021sinkhorn,
  title={Sinkhorn distributionally robust optimization},
  author={Wang, Jie and Gao, Rui and Xie, Yao},
  journal={arXiv preprint arXiv:2109.11926},
  year={2021}
}

@article{qi2021online,
  title={An online method for a class of distributionally robust optimization with non-convex objectives},
  author={Qi, Qi and Guo, Zhishuai and Xu, Yi and Jin, Rong and Yang, Tianbao},
  journal={Advances in Neural Information Processing Systems},
  volume={34},
  pages={10067--10080},
  year={2021}
}

@article{qi2022stochastic,
  title={Stochastic constrained dro with a complexity independent of sample size},
  author={Qi, Qi and Lyu, Jiameng and Bai, Er Wei and Yang, Tianbao and others},
  journal={arXiv preprint arXiv:2210.05740},
  year={2022}
}

@inproceedings{zhang2024large,
  title={Large-scale non-convex stochastic constrained distributionally robust optimization},
  author={Zhang, Qi and Zhou, Yi and Prater-Bennette, Ashley and Shen, Lixin and Zou, Shaofeng},
  booktitle={Proceedings of the AAAI Conference on Artificial Intelligence},
  volume={38},
  number={8},
  pages={8217--8225},
  year={2024}
}

@inproceedings{zhang2025improved,
  title={Improved rates of differentially private nonconvex-strongly-concave minimax optimization},
  author={Zhang, Ruijia and Lei, Mingxi and Ding, Meng and Xiang, Zihang and Xu, Jinhui and Wang, Di},
  booktitle={Proceedings of the AAAI Conference on Artificial Intelligence},
  volume={39},
  number={21},
  pages={22524--22532},
  year={2025}
}

@inproceedings{yang2022differentially,
  title={Differentially private sgda for minimax problems},
  author={Yang, Zhenhuan and Hu, Shu and Lei, Yunwen and Vashney, Kush R and Lyu, Siwei and Ying, Yiming},
  booktitle={Uncertainty in Artificial Intelligence},
  pages={2192--2202},
  year={2022},
  organization={PMLR}
}

@article{jin2021non,
  title={Non-convex distributionally robust optimization: Non-asymptotic analysis},
  author={Jin, Jikai and Zhang, Bohang and Wang, Haiyang and Wang, Liwei},
  journal={Advances in Neural Information Processing Systems},
  volume={34},
  pages={2771--2782},
  year={2021}
}

@inproceedings{blitzer2006domain,
  title={Domain adaptation with structural correspondence learning},
  author={Blitzer, John and McDonald, Ryan and Pereira, Fernando},
  booktitle={Proceedings of the 2006 conference on empirical methods in natural language processing},
  pages={120--128},
  year={2006}
}

@article{sagawa2019distributionally,
  title={Distributionally robust neural networks for group shifts: On the importance of regularization for worst-case generalization},
  author={Sagawa, Shiori and Koh, Pang Wei and Hashimoto, Tatsunori B and Liang, Percy},
  journal={arXiv preprint arXiv:1911.08731},
  year={2019}
}

@article{goodfellow2014explaining,
  title={Explaining and harnessing adversarial examples},
  author={Goodfellow, Ian J and Shlens, Jonathon and Szegedy, Christian},
  journal={arXiv preprint arXiv:1412.6572},
  year={2014}
}

@article{sinha2017certifying,
  title={Certifying some distributional robustness with principled adversarial training},
  author={Sinha, Aman and Namkoong, Hongseok and Volpi, Riccardo and Duchi, John},
  journal={arXiv preprint arXiv:1710.10571},
  year={2017}
}

@article{madry2017towards,
  title={Towards deep learning models resistant to adversarial attacks},
  author={Madry, Aleksander and Makelov, Aleksandar and Schmidt, Ludwig and Tsipras, Dimitris and Vladu, Adrian},
  journal={arXiv preprint arXiv:1706.06083},
  year={2017}
}

@article{ben2013robust,
  title={Robust solutions of optimization problems affected by uncertain probabilities},
  author={Ben-Tal, Aharon and Den Hertog, Dick and De Waegenaere, Anja and Melenberg, Bertrand and Rennen, Gijs},
  journal={Management Science},
  volume={59},
  number={2},
  pages={341--357},
  year={2013},
  publisher={INFORMS}
}

@article{shapiro2017distributionally,
  title={Distributionally robust stochastic programming},
  author={Shapiro, Alexander},
  journal={SIAM Journal on Optimization},
  volume={27},
  number={4},
  pages={2258--2275},
  year={2017},
  publisher={SIAM}
}

@article{rahimian2019distributionally,
  title={Distributionally robust optimization: A review},
  author={Rahimian, Hamed and Mehrotra, Sanjay},
  journal={arXiv preprint arXiv:1908.05659},
  year={2019}
}

@article{dwork2014algorithmic,
  title={The algorithmic foundations of differential privacy},
  author={Dwork, Cynthia and Roth, Aaron and others},
  journal={Foundations and trends{\textregistered} in theoretical computer science},
  volume={9},
  number={3--4},
  pages={211--407},
  year={2014},
  publisher={Now Publishers, Inc.}
}

@inproceedings{zhang2025revisiting,
  title={Revisiting large-scale non-convex distributionally robust optimization},
  author={Zhang, Qi and Zhou, Yi and Khan, Simon and Prater-Bennette, Ashley and Shen, Lixin and Zou, Shaofeng},
  booktitle={The Thirteenth International Conference on Learning Representations},
  year={2025}
}

@inproceedings{hsieh2021limits,
  title={The limits of min-max optimization algorithms: Convergence to spurious non-critical sets},
  author={Hsieh, Ya-Ping and Mertikopoulos, Panayotis and Cevher, Volkan},
  booktitle={International Conference on Machine Learning},
  pages={4337--4348},
  year={2021},
  organization={PMLR}
}

@article{zhang2023stochastic,
  title={Stochastic approximation approaches to group distributionally robust optimization},
  author={Zhang, Lijun and Zhao, Peng and Zhuang, Zhen-Hua and Yang, Tianbao and Zhou, Zhi-Hua},
  journal={Advances in Neural Information Processing Systems},
  volume={36},
  pages={52490--52522},
  year={2023}
}

@article{duchi2019variance,
  title={Variance-based regularization with convex objectives},
  author={Duchi, John and Namkoong, Hongseok},
  journal={Journal of Machine Learning Research},
  volume={20},
  number={68},
  pages={1--55},
  year={2019}
}

@article{duchi2021statistics,
  title={Statistics of robust optimization: A generalized empirical likelihood approach},
  author={Duchi, John C and Glynn, Peter W and Namkoong, Hongseok},
  journal={Mathematics of Operations Research},
  volume={46},
  number={3},
  pages={946--969},
  year={2021},
  publisher={INFORMS}
}

@article{dentcheva2017statistical,
  title={Statistical estimation of composite risk functionals and risk optimization problems},
  author={Dentcheva, Darinka and Penev, Spiridon and Ruszczy{\'n}ski, Andrzej},
  journal={Annals of the Institute of Statistical Mathematics},
  volume={69},
  number={4},
  pages={737--760},
  year={2017},
  publisher={Springer}
}

@article{selvi2304differential,
  title={Differential Privacy via Distributionally Robust Optimization. arXiv 2023},
  author={Selvi, A and Liu, H and Wiesemann, W},
  journal={arXiv preprint arXiv:2304.12681}
}

@article{balle2018privacy,
  title={Privacy amplification by subsampling: Tight analyses via couplings and divergences},
  author={Balle, Borja and Barthe, Gilles and Gaboardi, Marco},
  journal={Advances in neural information processing systems},
  volume={31},
  year={2018}
}
